\documentclass{article}

\PassOptionsToPackage{numbers, compress}{natbib}

 \usepackage[preprint]{neurips_2026}

\usepackage[utf8]{inputenc} %
\usepackage[T1]{fontenc}    %
\usepackage{url}            %
\usepackage{microtype}      %
\usepackage[dvipsnames, table]{xcolor}

\definecolor{darkblue}{rgb}{0, 0, 0.5}
\usepackage[colorlinks=true, citecolor=darkblue,linkcolor=darkblue, urlcolor=darkblue]{hyperref}

\usepackage{fancyhdr}
\usepackage{algorithm}
\usepackage{algorithmic}
\usepackage{forloop}

\usepackage{booktabs}
\usepackage{makecell}
\usepackage{multirow}
\usepackage{arydshln}

\usepackage{graphicx}
\usepackage{comment}
\usepackage{xspace}
\usepackage{enumitem}
\usepackage{pgfplots}
\usepgfplotslibrary{groupplots}

\usepackage{tikz}
\usetikzlibrary{arrows.meta,positioning,calc,decorations.pathreplacing,shapes.geometric,backgrounds,fit,patterns,fadings}

\usepackage{subcaption}

\usepackage{amsmath}
\usepackage{amssymb}
\usepackage{mathtools}
\usepackage{amsthm}
\usepackage{amsfonts}
\usepackage{bm}
\usepackage{nicefrac}
\usepackage{dsfont}

\usepackage{mdframed}
\definecolor{theoremcolor}{rgb}{0.94, 0.94, 0.94}
\definecolor{examplecolor}{rgb}{1, 1, 1.0}
\mdfsetup{
    backgroundcolor=theoremcolor,
    linewidth=0pt,
}
\newmdtheoremenv[linewidth=0pt,innerleftmargin=4pt,innerrightmargin=4pt]{definition}{Definition}
\newmdtheoremenv[linewidth=0pt,innerleftmargin=4pt,innerrightmargin=4pt]{proposition}{Proposition}
\newmdtheoremenv[linewidth=0pt,innerleftmargin=0pt,innerrightmargin=0pt,backgroundcolor=examplecolor]{example}{Example}
\newmdtheoremenv[linewidth=0pt,innerleftmargin=4pt,innerrightmargin=4pt]{corollary}{Corollary}
\newmdtheoremenv[linewidth=0pt,innerleftmargin=4pt,innerrightmargin=4pt]{theorem}{Theorem}
\newmdtheoremenv[linewidth=0pt,innerleftmargin=4pt,innerrightmargin=4pt]{lemma}{Lemma}
\newenvironment{restatetheorem}[2][]{%
  \par\medskip
  \begin{mdframed}[
    backgroundcolor=theoremcolor,
    linewidth=0pt,
    innerleftmargin=4pt,
    innerrightmargin=4pt,
    nobreak=true
  ]%
  \noindent
  \textbf{#2}%
  \ifx\relax#1\relax
  \else
    \ \textnormal{(#1)}%
  \fi
  \textbf{. }\itshape
}{%
  \end{mdframed}
  \medskip
}

\DeclareMathOperator*{\argmax}{arg\,max}
\DeclareMathOperator*{\argmin}{arg\,min}

\newcommand{\simplex}{\triangle}

\makeatletter
\def\adl@drawiv#1#2#3{%
        \hskip.5\tabcolsep
        \xleaders#3{#2.5\@tempdimb #1{1}#2.5\@tempdimb}%
                #2\z@ plus1fil minus1fil\relax
        \hskip.5\tabcolsep}
\newcommand{\cdashlinelr}[1]{%
  \noalign{\vskip 2pt
           \global\let\@dashdrawstore\adl@draw
           \global\let\adl@draw\adl@drawiv}
  \cdashline{#1}[.4pt/2pt]
  \noalign{\global\let\adl@draw\@dashdrawstore
           \vskip 2pt}}
\makeatother

\usepackage{multirow}
\usepackage{pifont}
\usepackage{wrapfig}

\NewDocumentCommand{\entmax}{o}{%
  \IfNoValueTF{#1}%
    {%
      \ensuremath{\alpha\text{-entmax\xspace}}%
    }{%
      \ensuremath{#1\text{-entmax\xspace}}%
    }%
    \xspace
}

\newcommand{\methodname}{DashAttention\xspace}
\newcommand{\methoddescription}{{Differentiable and Adaptive Sparse Hierarchical Attention}\xspace}

\title{\methodname: Differentiable and Adaptive \\ Sparse Hierarchical Attention}

\author{
\textbf{Yuxiang Huang\textsuperscript{1}\thanks{Equal contribution.}}~,~
\textbf{Nuno M. T. Gon\c{c}alves\textsuperscript{2,3,4}\footnotemark[1]}~,~
\textbf{Federico Alvetreti\textsuperscript{5}},~
\textbf{Lei Li\textsuperscript{4}},~\\
\textbf{Xu Han\textsuperscript{1}},~
\textbf{Edoardo M. Ponti\textsuperscript{6}},~
\textbf{Andr\'e F. T. Martins\textsuperscript{2,3,7,8}},~
\textbf{Marcos V. Treviso\textsuperscript{2,3,7}}
\\
\textsuperscript{1}Tsinghua University.
\textsuperscript{2}Instituto Superior T\'ecnico, Universidade de Lisboa. \\
\textsuperscript{3}Instituto de Telecomunica\c{c}\~oes.
\textsuperscript{4}Carnegie Mellon University. 
\textsuperscript{5}Sapienza University of Rome.
\\
\textsuperscript{6}University of Edinburgh.
\textsuperscript{7}TransPerfect.
\textsuperscript{8}ELLIS Unit Lisbon.
\\
\texttt{huang-yx21@mails.tsinghua.edu.cn, nuno.m.goncalves@ulisboa.pt}
}

\begin{document}
\setcounter{footnote}{0}

\maketitle

\begin{abstract}
Current hierarchical attention methods, such as NSA and InfLLMv2, select the top-$k$ relevant key-value (KV) blocks based on coarse attention scores and subsequently apply fine-grained softmax attention on the selected tokens. 
However, the top-$k$ operation assumes the number of relevant tokens for any query is fixed and it precludes the gradient flow between the sparse and dense stages. 
In this work, we propose \methodname ({\methoddescription}), which leverages the adaptively sparse $\alpha$-entmax transformation to select a \emph{variable} number of blocks according to the current query in the first stage. 
This in turn provides a prior for the second-stage softmax attention, keeping the entire hierarchy \emph{fully differentiable}.
Contrary to other hierarchical attention methods, we show that \methodname is non-dispersive, translating to better long-context modeling ability. 
Experiments with large language models (LLMs) show that \methodname achieves comparable accuracy as full attention with 75\% sparsity and a better Pareto frontier than NSA and InfLLMv2, especially in high-sparsity regimes. 
We also provide an efficient, GPU-aware implementation of \methodname in Triton, which achieves a speedup of up to $3.3\times$ over FlashAttention-3 at inference time.
Overall, \methodname offers a cost-effective strategy to model long contexts.\footnote{Code available at \url{https://github.com/fasa-org/dash-attention}.}
\end{abstract}

\section{Introduction}
\label{sec:intro}

The complexity of long-context tasks is impacted by the quantity of information to be retrieved, its obfuscation (i.e., whether it is distinguishable from noise), and its distribution throughout the context (e.g., scattered or concentrated) \citep{goldman-etal-2024-really,nawrot2026sparsefrontiersparseattention}.
Hence, to achieve good performance in such tasks, models must remain \emph{selective} enough to ignore  irrelevant tokens, but also \emph{sufficiently flexible} to recover the positions that matter for a given query, irrespective of their number, position, or similarity to the rest of the context.
Dense softmax attention~\citep{vaswani2017attention} handles the second requirement but fails the first, since every visible token receives nonzero probability mass, which can cause dispersion \citep{velickovic2025softmax}. 
Hard sparse routing methods such as top-$k$ block selection~\citep{yuan-etal-2025-native,zhao2026infllmv} address the first requirement; however, they do so by imposing a fixed budget, sacrificing the second requirement, and by severing the differentiable path between coarse routing decisions and fine token-level attention.
Therefore, achieving both query-dependent flexibility and stringent token-level selectivity remains an open challenge.

This paper proposes \methodname (\methoddescription), a multi-stage attention mechanism designed to meet both requirements.
Instead of routing with a hard top-$k$ operator over block scores, we route with $\alpha$-entmax~\citep{peters-etal-2019-sparse}, an adaptive sparse distribution whose support is learned from the input itself and whose nonzero masses remain differentiable.
The chunk-level probabilities are then consumed by a second online-softmax refinement stage that operates at full resolution, but only on tokens inside the routed chunks.
This way, the model learns \emph{where} and \textit{how much} to look at coarsely, and \emph{what} to read precisely, while keeping the entire hierarchy differentiable end-to-end.

This design is motivated by both theory and practice.
On the theory side, recent work has shown that softmax full attention can be dispersive in long contexts, whereas sparse alternatives such as \entmax can preserve concentration and improve long-context capabilities~\citep{velickovic2025softmax,vasylenko2026longcontext}.
On the systems side, modern sparse attention methods almost always rely on a hierarchical decomposition, since selecting tokens based directly on attention scores is too expensive: the query-key multiplications must still be computed, leaving little opportunity to reduce computation~\citep{yuan-etal-2025-native,zhao2026infllmv,lu2025moba}.
Our key intuition is to combine these two ideas into a single hierarchical attention mechanism that is inherently adaptive, fully differentiable, and efficient. 
Our main contributions are as follows:
\begin{itemize}[leftmargin=1em, itemsep=1pt, topsep=0pt]
\item We analyze the limitations of top-$k$-based sparse attentions and introduce \methodname, an end-to-end differentiable method that adaptively allocates sparsity across attention heads.
\item We integrate \methodname into long-context continual pretraining and show that it outperforms existing hierarchical sparse attention methods while matching full-attention accuracy.
\item We present an efficient GPU-aware implementation of \methodname that achieves speedups of 3.36$\times$ and 1.35$\times$ over FlashAttention-3 and InfLLMv2, respectively.
\end{itemize}

\begin{figure}[t]
    \centering
    \includegraphics[width=1\linewidth]{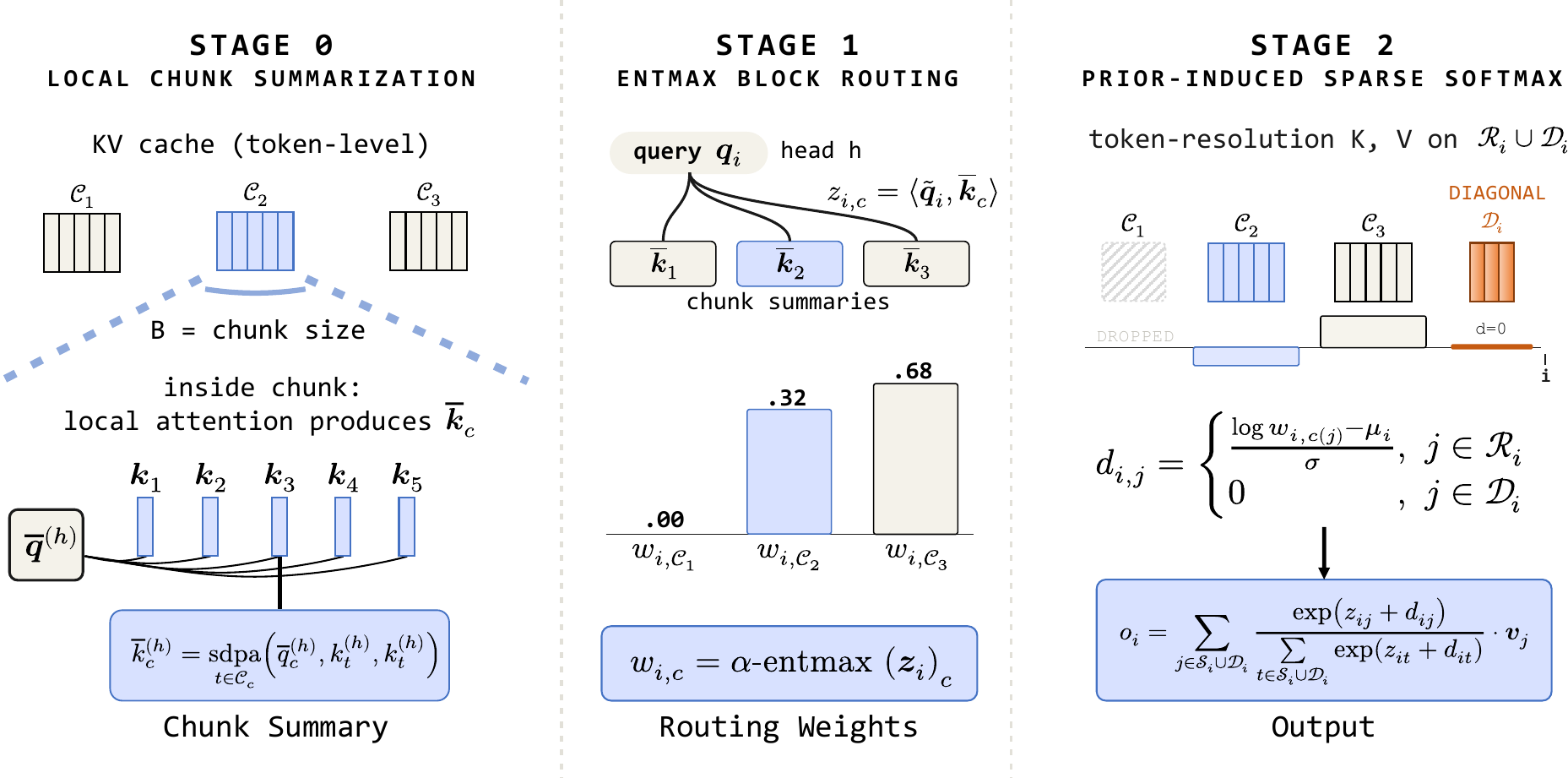}
    \caption{High-level overview of \methodname. Stage 0 builds chunk summaries by local SDPA; Stage 1 routes with \entmax to obtain an adaptively sparse support; Stage 2 refines at token resolution with a softmax whose logits are biased by $d_{i,j}$, derived from the router weights, maintaining full differentiability and FlashAttention compatibility. See Section~\ref{sec:method} for details.}
    \label{fig:method}
\end{figure}

\section{Background}
\label{sec:background}

\paragraph{Standard scale dot-product attention.}

Given a set of matrices $\bm{Q}, \bm{K},\bm{V} \in \mathbb{R}^{n \times d}$ containing $d$-dimensional representations for $n$ queries, keys, and values, the \textit{scale dot-product
self-attention} (SDPA) at a single head is computed as follows~\citep{vaswani2017attention}:
\begin{equation}\label{eq:dotproduct-attention}
    \bm{Z} = \frac{\bm{Q}\bm{K}^\top}{\sqrt{d}} \in \mathbb{R}^{n \times n}, \quad
    \bm P = \pi \left( \bm Z + \bm M \right), \quad
    \bm{O} = \bm P \bm{V} \in \mathbb{R}^{n \times d},
\end{equation}
where $\bm M$ is the attention mask and $\pi: \mathbb{R}^n\to \triangle_{n}$ is a transformation applied row-wise that maps logits to  vectors in the $n$-dimensional probability simplex $\triangle_n := \{\bm{p} \in \mathbb{R}^n \,:\, \mathbf{1}^\top \bm{p} = 1, \, \bm{p}\ge \mathbf{0}\}$, with softmax being the most common choice.
FlashAttention~\citep{dao2022flashattention} efficiently implements Eq.~\eqref{eq:dotproduct-attention} by reducing memory to $\mathcal{O}(n)$ while preserving softmax semantics.

\paragraph{Adaptively sparse attention: \entmax.}
\label{subsec:entmax}

Softmax attention is inherently dense, as every key receives a positive weight. A principled \emph{differentiable} sparse alternative is the \entmax function~\citep{peters-etal-2019-sparse}, 
\begin{equation}\label{eq:solution_entmax}
    \entmax(\bm{s}) = [(\alpha - 1)\bm{s} - \tau\bm{1}]_{+}^{\frac{1}{\alpha-1}},
\end{equation}
where $[\cdot]_{+}$ is the ReLU function, and $\tau \in \mathbb{R}$ is the (unique) normalizing constant which ensures the output is a valid probability distribution, $\sum_i \entmax (\bm{s})_i = 1$. 
Importantly, for $\alpha>1$, \entmax can produce sparse probability vectors. 
When $\alpha\to1$, \entmax recovers softmax, whereas $\alpha=2$ leads to sparsemax~\citep{martins2016softmax}; 
the sparsity increases monotonically with $\alpha$. 
This means that coordinates with $(\alpha-1)s_i\le \tau$ become exactly zero.
As such, 
\entmax yields dynamic sparsity, where the pattern of zeros (as well as its number) depends on the input $\bm{s}$.
Efficient GPU implementations for \entmax attention have been developed in AdaSplash~\citep{goncalves2025adasplash,goncalves2026adasplash2}.
More information on the \entmax transformation, including the distinction to top-$k$ softmax, is provided in App.~\ref{section:app_alpha_entmax}.

\paragraph{Hierarchical trainable sparse attention.}

Recently proposed pretraining-compatible sparse attention methods~\citep{yuan-etal-2025-native, zhao2026infllmv, lu2025moba} typically adopt a coarse-to-fine hierarchical design: the context is divided into blocks or pages, summarized by lightweight representations, routed via top-$k$ coarse relevance scores, and followed by exact token-level SDPA within the selected KV segments.
However, their fixed top-$k$ selection overlooks the heterogeneous information demands across tokens, attention heads, and layers.
As shown in Sec.~\ref{subsec:ablation_studies}, \methodname dynamically allocates sparsity across tokens and heads, which is crucial for improving performance under a fixed sparsity budget.

\section{\methodname}
\label{sec:method}

\paragraph{Motivation.}
The key innovation of \methodname is that the coarse router is itself a sparse attention mechanism---specifically, $\alpha$-entmax from Eq.~\eqref{eq:solution_entmax}---rather than a top-$k$ truncation applied to dense scores. Since $\alpha$-entmax is fully differentiable, the hierarchical attention remains end-to-end trainable, %
while ensuring a level of sparsity adaptive to each specific token sequence. Hence, 
\methodname is built around two complementary inductive biases: 
$\alpha$-entmax is used %
to suppress irrelevant chunks, whereas
softmax is used %
to obtain a relative ranking inside small selected regions. %
The resulting attention is therefore \textit{sparse} where sparsity is computationally useful, and \textit{dense} where density is semantically meaningful.
This combination also matters for optimization.
In methods based on top-$k$ block routing, the budget $k$ itself becomes a central hyperparameter, and the support changes through a discontinuous top-$k$ map; thus, gradients do not directly capture how routing scores should change block membership.
In \methodname, by contrast, the support is learned from the data, the coarse masses are differentiable, and the refinement stage preserves a direct gradient path from the final loss to the block-summarization (Stage 0) and block-routing (Stage 1) parameters.

\paragraph{Problem setup and notation.}
Assume the model has model dimension $d$ and head dimension $d_h$. We let $h_q$ and $h_{kv}$ denote the number of query and key-value heads, respectively. 
When using GQA~\citep{ainslie2023gqa}, we denote the query-group of the $r$\textsuperscript{th} KV head as $\mathcal{G}_r$, and the grouping factor as $g_{q}$.
We partition the prefix into contiguous chunks of size $B$, and denote the index set of chunk $c$ by $\mathcal{C}_c$.
For a query position $i$, \methodname proceeds in three stages, illustrated in Fig.~\ref{fig:method}: 
(Stage~0)~we first summarize each chunk into a compact representation using local attention driven by a learned summarization head;
(Stage~1)~the query attends to these chunk summaries with \entmax, producing a sparse routing distribution; (Stage~2)~only the routed chunks are expanded back to token resolution and refined with softmax.
We explain each stage next.

\subsection{Stage 0: Local Chunk Summarization}
\label{sec:stage0}

Existing methods use mean pooling~\citep{zhao2026infllmv, lu2025moba} or a learned MLP~\citep{yuan-etal-2025-native} to summarize each KV chunk. However, they are either less expressive or hard to adapt from a pretrained model.
To mitigate that, we introduce a local-attention mechanism by using a learned query vector $\bar{\bm q}\in \mathbb{R}^{h_{kv}\times d_h}$ initialized at zero. As a result, the model begins from a conservative mean-pooling regime and only gradually learns more expressive chunk summaries.
During inference, we compute a summary vector for each chunk by a local SDPA as follows, where $r \in \{1, \ldots, h_{kv}\}$ indexes the key-value head:
\begin{equation}
\bar{\bm{k}}_{c}^{(r)} = \sum_{t \in \mathcal{C}_c} \frac{\exp\!\left(\langle \bar{\bm{q}}^{(r)},\, {\bm{k}}_{t}^{(r)} \rangle\, /\, \sqrt{d_h}\right)}{\sum_{u \in \mathcal{C}_c} \exp\!\left(\langle \bar{\bm{q}}^{(r)},\, {\bm{k}}_{u}^{(r)} \rangle\, /\, \sqrt{d_h}\right)}\; {\bm{k}}_{t}^{(r)}\,,
\label{eq:chunk_rep}
\end{equation}

This design is more expressive than mean pooling used in MoBA and InfLLMv2, as it can be interpreted as a learned summary head that reads each chunk from within the chunk itself.
It is also easier to adapt compared with MLPs used in NSA.
At the start of training, the inner products in Eq.~\eqref{eq:chunk_rep} are zero, causing the local softmax in Eq.~\eqref{eq:chunk_rep} to reduce to uniform mean pooling.
This enables a smooth transition from mean pooling to weighted average of the key vectors within the chunk. 
Notably,
chunk summaries become fixed once a chunk has been fully generated, since the $\bar{\bm{k}}_{c}$ only depends on $\bm k$'s from positions in $\mathcal{C}_c$.
Therefore, this design also guarantees no additional recomputation during inference, aligned with previous chunk summarization methods.
Ablation studies on local attention are conducted in App.~\ref{sec:supp-local-attn-sigma}.

\subsection{Stage 1: Entmax Chunk Routing}
\label{sec:stage1}

Given each query head $h$ associated with key-value head $r = \lfloor h/g_q \rfloor$, we compute chunk-level logits 
$\bar z_{i,c}^{(h)} =  \langle {\bm{q}}_{i}^{(h)},\, \bar{\bm{k}}_{c}^{(r)} \rangle\, /\, \sqrt{d_h}$ for all causal chunks $c$ visible from position $i$. 
These logits are then transformed into a sparse chunk distribution with \entmax with a 
scaling factor $\gamma$:
\begin{equation}
\hat{\bm{w}}_{i}^{(h)} = \alpha\text{-entmax}\!\left(\gamma \bar{\bm{z}}_{i}^{(h)} \right) \in \triangle_{\lfloor n/B\rfloor },\quad h\in \{1, 2, \cdots, h_q\}.
\label{eq:entmax_routing}
\end{equation}
The support $\hat{\mathcal{S}}_{i}^{(h)} = \{c \mid w_{i,c}^{(h)} > 0\}$ determines chunk positions that are retained for the next stage, while other chunks are discarded.
Because routing is performed with \entmax rather than top-$k$, the support is not determined by a fixed budget. 
Instead, it is induced by the geometry of the chunk scores themselves. As a result, informative queries may route to many chunks, while sharp queries may route to very few, which enables dynamic sparsity allocation among different tokens, heads, and layers which is one of the defining advantages of \methodname over fixed-budget sparse attention.

We note that GQA is commonly used in modern LLMs. Therefore, when $h_q > h_{kv}$, multiple query heads share the same key-value head, so Stage~1 must produce a shared routed support for each GQA group. 
In \methodname, for each group $r\in\{1,\cdots,h_{kv}\}$ with member query heads $\mathcal{G}_r$, we average the head-wise probability scores: $\bm{w}_i^{(r)} = \sum_{h \in \mathcal{G}_r} \hat{\bm w_i}^{(h)} / g_q \in \triangle_{\lfloor n/B\rfloor}$, which remains sparse and with a support $\mathcal{S}_i^{(r)} = \cup_{h\in \mathcal{G}_r} \hat{\mathcal{S}}_i^{(h)}$ determined jointly by the chunk-score geometry within $\mathcal{G}_r$.

\subsection{Stage 2: Prior-Induced Sparse Softmax Attention}
\label{sec:stage2}

Here, we derive the final stage of~\methodname. As mentioned in Sec.~\ref{sec:intro}, our goal is to make~\methodname both adaptively sparse and readily adaptable from softmax attention.
To this end, we start from the variational form of softmax~\citep{blondel2019learning}:
\begin{equation}
\text{softmax}(\bm z) = \argmax_{\bm p \in \triangle_{n}} \bm z^\top \bm p - \operatorname{KL}(\bm p \,\|\, \bm u),
\end{equation}
where $\bm u$ denotes an uniform probability vector and $\bm{z}$ the attention logits where $z_{i,j} = \langle \bm q_i, \bm k_j\rangle/\sqrt{d_h}$ and we omit the notation for attention heads for simplicity. 
Instead, we carry the knowledge from Stage~1 ($\bm{w}$) by varying the reference distribution in the $\operatorname{KL}$ term from $\bm u$ to another distribution $\bm g \in \triangle_{n}$, which is obtained from the block routing scores $\bm w$ via $\bm g = g_\sigma(\bm w)$ and described in detail in Sec.~\ref{sec:diagonal}.
We therefore obtain a token-grained sparse probability distribution that enables a block-sparse differentiable softmax computation. Given the Stage~1 scores of the $i$\textsuperscript{th} token $\bm w_i$, we compute the attention probability at the $i$\textsuperscript{th} token as follows:
\begin{equation}
\bm p^{\text{route}}_i = \argmax_{\bm p\in \triangle_{n}} \bm z^\top_i \bm p - \operatorname{KL} \left(\bm p \,\|\, g_\sigma(\bm w_i) \right),\quad p_{i,j}^{\text{route}} = \frac{ g_{i,j} \exp(z_{i,j})}{\sum_{t\leq i} g_{i,t}\exp(z_{i,t})},
\end{equation}
where ${g}_{i,j} = g_\sigma(\bm{w}_i)_j$.
By requiring $g_{i,j} = 0$ whenever $w_{i,j} = 0$ with an appropriate construction of $g_\sigma$, the resulting softmax attention becomes naturally sparse,
since if $g_{i,j}\to 0^+$, $-\operatorname{KL}(p_j\|g_{i,j})\to+\infty$ for any $p_j> 0$, thereby masking Stage~2's softmax probability with an entmax-computed dynamic mask.
At the same time, the properties of dynamic sparsity and full differentiability are preserved through the entmax score $\bm w$. 
The final $i$\textsuperscript{th} output becomes 
\begin{equation}
\label{eq:stage2}
    \bm o_i = \bm V^\top \bm p_i^{\text{route}} = \sum_{j\in \mathcal{S}^{\text{route}}_i} \frac{g_{i,j} \exp(z_{i,j})}{\sum_{t\in \mathcal{S}^{\text{route}}_i} g_{i,t} \exp(z_{i,t})}\cdot \bm v_j,
\end{equation}
where $\mathcal{S}^{\text{route}}_i = \{c \mid g_{i,c} > 0\}$ is shared within the same query group to enable similar matrix multiplication tiling layout with NSA~\citep{yuan-etal-2025-native} and InfLLMv2~\citep{zhao2026infllmv}.
Note that this form is readily compatible with existing efficient attention kernels. In practice, we employ AdaSplash-2~\citep{goncalves2026adasplash2} for Stage 1 (\entmax attention) and FlashAttention~\citep{dao2023flashattention2} for Stage 2 (softmax attention), enabling efficient training and inference. Moreover, Eq.~\eqref{eq:stage2} 
is fully differentiable with respect to $\bm q$, $\bm k$, $\bm v$, and the prior score $\bm g$. 
As a result, in \methodname gradients can propagate back to Stage 1's \entmax scores $\bm w$, thereby making the local chunk summarization from Stage 0 trainable.

\subsection{Controlling Prior Strength and Handling Diagonal Chunks}
\label{sec:diagonal}

Since the diagonal or near-diagonal region may contain fewer than $B$ tokens, the Stage 1 priors $\bm{g} = g_\sigma(\bm{w})$ cannot be obtained due to the absence of block summaries.
At the same time, replacing $\operatorname{KL}(\bm p \,\|\, \bm u)$ with $\operatorname{KL}(\bm p \,\|\, \bm g)$ would drift the outcome $\bm p_i^{\text{route}}$ to a coarsely estimated $\bm g$, since the summary vectors provide only an approximation of each block.
Therefore, to properly estimate $\bm{g}$,
we split the attention mass of $\bm{w}$ into two branches: a routed branch $\mathcal{R}_i$ over blocks selected in Stage~1, and a diagonal branch $\mathcal{D}_i$ that requires prior estimation. 
To control the prior strength, we introduce a hyperparameter $\sigma$ and construct a strength-reduced prior probability as follows:
\begin{equation}
w'_{i,j} = \frac{1}{B}\cdot \frac{w_{i,c(j)}^{\nicefrac{1}{\sigma}}} { \bm 1^\top \bm {w}_i^{ \nicefrac{1}{\sigma}}},
\end{equation}
where $c(j)$ represents the chunk associated with the $j$\textsuperscript{th} element.
As $\sigma \to \infty$, the prior gradually becomes uniform over the routed support: $w_{i,j}' \to 1/|\mathcal{R}_i|$ for $j \in \mathcal{R}_i$, while $w_{i,j}' =0$ for $j \notin \mathcal{S}_i$, thereby corresponding to only applying softmax on selected tokens without any additional prior from Stage~1. 
With that in mind, we define the prior transformation $g_\sigma: \simplex_{\lfloor n/B \rfloor} \to \simplex_n$ as follows:
\begin{equation}
g_\sigma(\bm{w}_i)_j = g_{i,j} =
\begin{cases}
\lambda_iw_{i,j}' \,, & j \in \mathcal{R}_i, \\[0.5em]
\frac{1}{|\mathcal{D}_i|}(1 - \lambda_i)\,, & j \in \mathcal{D}_i, \\[0.5em]
0, & \text{otherwise.}
\end{cases}=
\begin{cases}
\frac{\lambda_iw_{i,c(j)}^{{\nicefrac{1}{\sigma}}}}{B\cdot \bm 1^\top \bm w_i^{\nicefrac{1}{\sigma}}} \,, & j \in \mathcal{R}_i, \\[0.5em]
\frac{1}{|\mathcal{D}_i|}(1 - \lambda_i)\,, & j \in \mathcal{D}_i, \\[0.5em]
0, & \text{otherwise.}
\end{cases}
\label{eq:lambda_factorization}
\end{equation}

Note that $g_{\sigma}(\bm w_i)$ is a probability vector and adds up to one.
By this construction, the prior assigns probability mass $\lambda_i$ to $\mathcal{R}_i$ and $1-\lambda_i$ to $\mathcal{D}_i$.
In practice, we choose $\lambda_i$ as a function of how informative the routing distribution is relative to a uniform baseline.
Let $\bm{u}_{\mathcal{R}_i}$ be the uniform distribution over the routed support and $\bm{w}'_{\mathcal{R}_i}$ denotes the probabilities of $\bm{w}'_i$ restricted to $\mathcal{R}_i$,
then:
\begin{equation}
\lambda_i = \operatorname{sigmoid}\!\left(\operatorname{KL}(\bm{u}_{{\mathcal{R}_i}} \,\|\, \bm{w}_{{\mathcal{R}_i}}') + \log \frac{|\mathcal{R}_i|}{|\mathcal{D}_i|}\right),
\label{eq:lambda_kl}
\end{equation}

This has an intuitive interpretation: if the block router is nearly uniform ($\operatorname{KL} \approx 0$), the model allocates mass roughly proportional to $|\mathcal{R}_i|/(|\mathcal{R}_i|+|\mathcal{D}_i|)$, whereas if the router is highly informative (large KL), more mass is given to the routed branch 
We also have $\lim_{\sigma\to+\infty}\lambda_i = |\mathcal{R}_i|/(|\mathcal{R}_i|+|\mathcal{D}_i|)$, showing that the priors can be weakened by utilizing a larger $\sigma$. 
Next, in Proposition~\ref{th:attn-bias}, we show a straightforward way to compute Eq.~\eqref{eq:lambda_factorization}  in softmax attention.
The proof is given in App.~\ref{sec:equiv-bias-form}.

\begin{proposition}
\label{th:attn-bias}
    Computing the attention output $\bm o_i$ with Eq.~\eqref{eq:stage2} and Eq.~\eqref{eq:lambda_factorization} is equivalent to first calculate $\mu_i = \operatorname{mean}_{j\in \mathcal{R}_i}\{\log w_{i, c(j)}\}$, followed by 
    $$
d_{i,j} =
\begin{cases}
\frac{\log w_{i,c(j)} - \mu_i}{\sigma} \,, & j \in \mathcal{R}_i, \\
0\,, & j \in \mathcal{D}_i. \\
\end{cases}\quad \bm o_i = \sum_{j\in \mathcal{R}_i\cup \mathcal{D}_i} \frac{\exp(z_{i,j} + d_{i,j})\cdot \bm v_j}{\sum_{t\in \mathcal{R}_i\cup \mathcal{D}_i} \exp(z_{i,t} + d_{i,t})}.
    $$
\end{proposition}

\subsection{GPU-Aware Implementation} 
\label{sec:gpu}
We implement \methodname in Triton~\citep{triton-paper}, consisting of three fused kernels that serve both prefill and incremental decoding, backed by a chunk-representation cache in HBM (High Bandwidth Memory). 
The three fused kernels (one per stage) are described below, with complete algorithms deferred to App.~\ref{sec:supp-algorithms}.

\textbf{Stage 0 kernel.}
A single Triton kernel lets the learned per-head summary token $\bar{\bm{q}}^{(r)}$ attend to its keys via an online softmax.
Because at this stage the keys are also used as values, the same loaded $\bm{K}$ tile will serve both purposes, so the chunk representation $\bar{\bm{k}}_c^{(r)}$ in Eq.~\eqref{eq:chunk_rep} is produced in a single pass with no extra HBM round-trip.
Once a chunk is complete, it is written back, persisted in the chunk representation cache, and reused across all subsequent query positions and decoding steps.

\textbf{Stage 1 kernel.}
The chunk logits $\bm{z}_i^{(h)}$ are formed by attending to the cached chunk representations.
Since the number of chunks $T_c = \lfloor n/B\rfloor$ is small (e.g., $T = 256$ for $n = 16$K with $B = 64$), the entire chunk-score row typically fits in registers, and the entmax threshold $\tau$ is solved in place by the AdaSplash-2 solver~\citep{goncalves2026adasplash2}.
For GQA, we then average the entmax probabilities across each group's query heads before pruning, keeping the support consistent within the KV group.
The sparse support $\mathcal{S}_i$ is stored as a bitpacked block mask $\bm{M}_i \in \{0,1\}^{h_{kv} \times T_c}$, packing 32 consecutive column chunks per \texttt{int32}, while the per-chunk routing bias $d_{i,j}$ is materialized alongside for use in Stage~2.

\textbf{Stage 2 kernel.}
A masked FlashAttention kernel iterates over the chunks marked active in $\bm{M}$, walking the bitpacked mask.
For each selected chunk, we add the per-chunk routing bias to the attention logits before the online softmax, in a single fused pass over $\mathcal{S}_i$, keeping the kernel fully compatible with FlashAttention.
For incremental decoding, a split-KV variant partitions $\bm{M}$ across thread blocks along the KV dimension, exposing enough parallelism to saturate the GPU for a single query~\citep{dao2022flashattention}.
For training, the backward pass can reuse $\bm{M}$ thanks to the sparse $\alpha$-entmax Jacobian~\citep{peters-etal-2019-sparse}.

\section{Theoretical Analysis}
\label{sec:theory}

As shown in~\citep{vasylenko2026longcontext, goncalves2025adasplash, goncalves2026adasplash2}, softmax attention suffers from \emph{dispersion} in long-context settings: the Shannon entropy of its attention distribution $\mathcal{H}(\bm p)$ grows as $\lim_{n\to\infty}\frac{\mathcal{H}(\bm p)}{\log n} = 1$ with sequence length $n$, 
making long-range dependency modeling increasingly difficult. 
Top-$k$ sparse attention mitigates this issue by bounding the entropy by $\log k$.
However, existing hierarchical sparse attention methods~\citep{yuan-etal-2025-native, zhao2026infllmv} combine GQA with post-softmax head aggregation before top-$k$ selection, allowing dispersion to re-emerge during aggregation. We formalize this property as follows:

\begin{definition}[Head aggregation]
\label{de:head-aggregation}
Given $f:\mathbb{R}^n\to \simplex_n$, and $H$ bounded sequence $\{\bm z^{(h)}\}_{h=1}^H\subset \mathbb{R}^n$, and the aggregation weights $\theta = (\theta_1, \cdots, \theta_H)^\top\in \simplex_{H}$, the aggregated probability is given by
$
\operatorname{aggr}_f\left(\bm z^{(1)}, \bm z^{(2)}, \cdots, \bm z^{(H)};\bm \theta\right) = \sum_{h=1}^H \theta_h\cdot f\left(\bm z^{(h)}\right)\in \simplex_n.
$
\end{definition}

Thus, softmax head aggregation is dispersive by construction, whereas the entmax head aggregation in \methodname is non-dispersive.

\begin{theorem}[Dispersion in head aggregation, Informal]
\label{th:dispersion-head-aggregation}
Under any finite $H$ and $\theta\in \simplex_{H}$:

1. Softmax head aggregation, i.e., $f=\operatorname{softmax}$, is dispersive; 

2. Denote $\bm p^{(h)} = \alpha\operatorname{-entmax}(\bm z^{(h)})$. If $\|\bm p^{(h)}\|_0 = \mathcal{O}\left(n^{\beta_h}\right)$ with $\beta_h\in (0, 1)$, then entmax head aggregation, i.e., $f=\alpha\operatorname{-entmax}$, is not dispersive.
\end{theorem}

Since softmax head aggregation is dispersive, it can undermine the non-dispersive property of top-$k$ sparse attention and lead to noisy selection. In contrast, \methodname performs entmax head aggregation directly on sparse entmax scores, mitigating this issue and perform better on challenging retrieval tasks, such as MK1--MK3 in Table~\ref{tab:results-ruler}. We provide a detailed proof in App.~\ref{sec:non-dispersion}.

\section{Experiments}
\label{sec:experiments}

We comprehensively evaluate \methodname from multiple perspectives. 
First, we assess its task performance across representative benchmarks 
(Sec.~\ref{subsec:performance}). 
Second, we analyze its computational efficiency (kernel inference speed) 
under different inference settings (Sec.~\ref{subsec:efficiency_bench}). 
Third, we examine the effectiveness--efficiency tradeoff induced by different sparsity configurations 
and compare the resulting Pareto frontiers against relevant baselines 
(Sec.~\ref{subsec:pareto-frontiers}). 
Finally, we study the dynamic sparsity allocation behavior of \methodname through ablations 
to better understand how it distributes computation across layers 
(Sec.~\ref{subsec:ablation_studies}).

\begin{table*}[t]
    \small
    
    \centering
    \setlength{\tabcolsep}{3.5pt} 
    \caption{Results on RULER with a 16K context length. Best results in bold.}
    \label{tab:results-ruler}
    \resizebox{\textwidth}{!}{
    \begin{tabular}{llccccccccccccccc}
    \toprule
    \multirow{2}{*}{{}} & \multirow{2}{*}{Method} & \multicolumn{13}{c}{RULER-16K} & \multirow{2}{*}{\makecell{Avg. $\uparrow$\\(\%)}} & \multirow{2}{*}{\makecell{Sparsity\\(\%)}} \\
    \cmidrule{3-15}
        &  & \textit{SG1} & \textit{SG2} & \textit{SG3} & \textit{MK1} & \textit{MK2} & \textit{MK3} & \textit{MV} & \textit{MQ} & \textit{VT} & \textit{CWE} & \textit{FWE} & \textit{QA1} & \textit{QA2} &  &  \\
    \midrule
    \rowcolor{black!5}
    \cellcolor{white}
    \multirow{4}{*}{\texttt{1B}} &
    FullAttn & 100.0 & 100.0 & 100.0 & 98.0 & 78.0 & 48.0 & 87.0 & 89.0 & 17.6 & 1.8 & 66.7 & 44.0 & 36.0 & 66.6 & 0.0\\
    & NSA & 98.0 & 92.0 & 86.0 & 60.0 & 20.0 & 6.0 & 56.0 & 53.5 & 25.6 & \textbf{1.6} & 60.7 & 32.0 & \textbf{36.0} & 48.3 & 75.0 \\
    & InfLLMv2 & \textbf{100.0} & \textbf{100.0} & \textbf{100.0} & 92.0 & 52.0 & 16.0 & \textbf{94.0} & \textbf{90.0} & 40.8 & 1.0 & 59.3 & \textbf{36.0} & 30.0 & 62.4 & 75.0 \\
    \rowcolor{white}
    \cellcolor{white} & \methodname & \textbf{100.0} & \textbf{100.0} & 96.0 & \textbf{94.0} & \textbf{70.0} & \textbf{26.0} & 89.5 & 87.5 & \textbf{46.0} & 0.8 & \textbf{63.3} & \textbf{36.0} & 34.0 & \textbf{64.9} & \textbf{75.7}
    \\
    \midrule
    \rowcolor{black!5}
    \cellcolor{white}
    \multirow{4}{*}{\texttt{3B}} &
    FullAttn & 100.0 & 100.0 & 98.0 & 100.0 & 94.0 & 60.0 & 96.0 & 96.5 & 20.4 & 5.8 & 19.3 & 60.0 & 50.0 & 69.2 & 0.0\\
    & NSA & \textbf{100.0} & 92.0 & 90.0 & 66.0 & 24.0 & 10.0 & 64.5 & 62.0 & 19.2 & 4.2 & 28.7 & 36.0 & 40.0 & 49.0 & 75.0 \\
    & InfLLMv2 & \textbf{100.0} & \textbf{100.0} & 98.0 & 94.0 & 56.0 & 28.0 & 95.5 & \textbf{97.0} & \textbf{22.8} & \textbf{5.0} & 16.0 & \textbf{56.0} & \textbf{48.0} & 62.8 & 75.0\\
    \rowcolor{white}
    \cellcolor{white} & \methodname & \textbf{100.0} & \textbf{100.0} & \textbf{100.0} & \textbf{100.0} & \textbf{88.0} & \textbf{40.0} & \textbf{96.5} & 95.5 & 21.2 & 1.6 & \textbf{42.7} & 52.0 & 42.0 & \textbf{67.7} & \textbf{75.4}\\
    \midrule
    \rowcolor{black!5}
    \cellcolor{white}
    \multirow{4}{*}{\texttt{8B}} &
    FullAttn & 100.0 & 100.0 & 100.0 & 98.0 & 100.0 & 96.0 & 100.0 & 99.5 & 58.4 & 49.2 & 78.0 & 78.0 & 52.0 & 85.3 & 0.0 \\
    & NSA & \textbf{100.0} & 82.0 & 68.0 & 60.0 & 34.0 & 12.0 & 73.5 & 74.5 & 30.0 & 42.8 & 68.7 & 36.0 & 34.0 & 55.0 & 75.0 \\
    & InfLLMv2 & \textbf{100.0} & \textbf{100.0} & \textbf{100.0} & \textbf{98.0} & 82.0 & 52.0 & 96.0 & 98.0 & 55.6 & 42.4 & \textbf{77.3} & 66.0 & \textbf{58.0} & 78.9 & 75.0 \\
    \rowcolor{white}
    \cellcolor{white}
    & \methodname & \textbf{100.0} & \textbf{100.0} & \textbf{100.0} & \textbf{98.0} & \textbf{96.0} & \textbf{86.0} & \textbf{100.0} & \textbf{99.0} & \textbf{60.0} & \textbf{49.6} & 72.0 & \textbf{72.0} & 54.0 & \textbf{83.6} & \textbf{75.7}\\
    \bottomrule
    \end{tabular}
    }
    \end{table*}

\begin{table*}[t]
    \small
        
    \centering
    \setlength{\tabcolsep}{7.5pt} 
    \caption{Results on HELMET with a 16K context length. Best results in bold.}
    \label{tab:results-helmet}
    \resizebox{\textwidth}{!}{
    \begin{tabular}{llccccccccc}
    \toprule
     & \multirow{2}{*}{Method} & \multicolumn{7}{c}{{{HELMET}-16K}} & \multirow{2}{*}{\makecell{Overall $\uparrow$\\(\%)}} & \multirow{2}{*}{\makecell{Sparsity\\(\%)}} \\
    \cmidrule(lr){3-9}
    & & Recall & ICL & Rerank & RAG & LongQA & Cite & Summ. & & \\
    \midrule
    \multirow{4}{*}{\texttt{1B}}
    & \cellcolor{black!5}FullAttn & \cellcolor{black!5}66.6 & \cellcolor{black!5}59.0 & \cellcolor{black!5}13.9 & \cellcolor{black!5}34.5 & \cellcolor{black!5}40.1 & \cellcolor{black!5}8.9 & \cellcolor{black!5}4.2 & \cellcolor{black!5}32.5 & \cellcolor{black!5}0.0 \\
    & NSA & 23.4 & 46.0 & 15.0 & 35.0 & 34.6 & 6.4 & 4.4 & 23.6 & 75.0 \\
    & InfLLMv2 & 43.7 & 59.2 & 15.5 & 37.0 & \textbf{38.7} & \textbf{7.6} & 4.9 & 29.5 & 75.0 \\
    & \cellcolor{white}\methodname & \cellcolor{white}\textbf{51.8} & \cellcolor{white}\textbf{61.0} & \cellcolor{white}\textbf{16.5} & \cellcolor{white}\textbf{38.5} & \cellcolor{white}36.4 & \cellcolor{white}7.4 & \cellcolor{white}\textbf{6.7} & \cellcolor{white}\textbf{31.2} & \cellcolor{white}\textbf{75.4} \\
    \midrule
    \multirow{4}{*}{\texttt{3B}}
    & \cellcolor{black!5}FullAttn & \cellcolor{black!5}78.6 & \cellcolor{black!5}44.8 & \cellcolor{black!5}25.3 & \cellcolor{black!5}48.5 & \cellcolor{black!5}42.5 & \cellcolor{black!5}12.3 & \cellcolor{black!5}9.7 & \cellcolor{black!5}37.4 & \cellcolor{black!5}0.0 \\
    & NSA & 29.9 & 45.2 & 20.5 & 45.5 & 39.7 & 7.3 & 7.1 & 27.9 & 75.0 \\
    & InfLLMv2 & 58.0 & 44.2 & \textbf{24.8} & \textbf{49.3} & \textbf{47.7} & 6.0 & \textbf{9.8} & 34.2 & 75.0 \\
    & \cellcolor{white}\methodname & \cellcolor{white}\textbf{66.9} & \cellcolor{white}\textbf{51.2} & \cellcolor{white}20.0 & \cellcolor{white}47.8 & \cellcolor{white}35.9 & \cellcolor{white}\textbf{10.3} & \cellcolor{white}8.3 & \cellcolor{white}\textbf{34.3} & \cellcolor{white}\textbf{75.4} \\
    \midrule
    \multirow{4}{*}{\texttt{8B}}
    & \cellcolor{black!5}FullAttn & \cellcolor{black!5}98.3 & \cellcolor{black!5}56.0 & \cellcolor{black!5}37.8 & \cellcolor{black!5}53.3 & \cellcolor{black!5}52.4 & \cellcolor{black!5}20.0 & \cellcolor{black!5}16.5 & \cellcolor{black!5}47.7 & \cellcolor{black!5}0.0 \\
    & NSA & 35.0 & 59.8 & 27.9 & 48.3 & 51.0 & 13.7 & \textbf{14.7} & 35.8 & 75.0 \\
    & InfLLMv2 & 81.4 & \textbf{63.6} & 36.8 & \textbf{54.0} & \textbf{53.2} & \textbf{18.3} & 13.7 & 45.9 & 75.0 \\
    & \cellcolor{white}\methodname & \cellcolor{white}\textbf{88.3} & \cellcolor{white}62.8 & \cellcolor{white}\textbf{39.7} & \cellcolor{white}\textbf{54.0} & \cellcolor{white}52.3 & \cellcolor{white}17.4 & \cellcolor{white}13.7 & \cellcolor{white}\textbf{46.9} & \cellcolor{white}\textbf{75.4} \\
    \bottomrule
    \end{tabular}
    }
    \end{table*}

\subsection{Task Performance}
\label{subsec:performance}

\paragraph{Baselines and Experimental Setup.} 
Besides full attention (FullAttn), we compare~\methodname with other GQA-based hierarchical softmax attention methods under long-context adaptation, namely
NSA~\citep{yuan-etal-2025-native} and InfLLMv2~\citep{zhao2026infllmv}, and follow the same long-context adaptation setting for all methods.
Concretely, we start with models pretrained with softmax full attention, continue pretraining them on 16K long-context data (InfLLMv2-5B dataset~\citep{openbmb_InfLLM_V2_data_5B_2025}), and then apply a short SFT stage (MiniCPM-4 Dataset~\citep{team2025minicpm4}). We use the 1B, 3B, and 8B variants of MiniCPM-4~\citep{team2025minicpm4} as base models. During training, we gradually increase $\alpha$ from $1.25$ to $1.5$, and use $\alpha=1.5$ for inference. For NSA and InfLLMv2 baseline models, we keep the sparsity to $75\%$ during training.
Full training details are in Appendix.~\ref{sec:long-context-adaption-training-appendix}.

\paragraph{Benchmark Datasets.}
We evaluate long-context performance on RULER~\citep{hsieh2024ruler} and HELMET~\citep{yen2025helmet}. 
Given the importance of matching sparsity ratios~\citep{nawrot2026sparsefrontiersparseattention}, for NSA and InfLLMv2, we set the attention sparsity to $75\%$ during both training and inference, whereas for~\methodname we tune the Stage~1 factor $\gamma$ to match a comparable sparsity level. 
To assess general capabilities, we further evaluate on MMLU~\citep{wang2024mmlu}, MMLU-Pro~\citep{hendrycks2020measuring}, CSQA~\citep{talmor2019commonsenseqa}, IFEval~\citep{zhou2023instruction}, HellaSwag~\citep{zellers2019hellaswag}, GSM8K~\citep{cobbe2021training}, MATH~\citep{hendrycksmath2021}, DROP~\citep{dua2019drop}, MBPP~\citep{austin2021program}, and HumanEval~\citep{chen2021codex}. Additional details are provided in App.~\ref{sec:long-context-adaption-benchmarks-appendix}. 
\looseness=-1

\begin{table*}[t]
    \centering
    \small
    \setlength{\tabcolsep}{4pt}
    \caption{General task performance with 8B models. Best results in bold.}
    \label{tab:general-tasks}
    \resizebox{\textwidth}{!}{
    \begin{tabular}{lccccccccccc}
      \toprule
      Method & MMLU & MMLU-Pro & CSQA & IFEval & HellaSwag & GSM8K & MATH & DROP & MBPP & HumanEval & Avg. \\
      \midrule
      \rowcolor{black!5}
      FullAttn & 73.6 & 46.1 & 80.0 & 80.0 & 72.9 & 76.3 & 22.2 & 11.3 & 63.0 & 69.5 & 59.5 \\
      InfLLMv2 & 73.2 & 45.6 & \textbf{80.6} & 79.0 & 72.8 & \textbf{77.1} & 22.2 & 10.3 & 62.4 & 67.7 & 59.1 \\
      NSA & 73.2 & \textbf{46.3} & 80.4 & 80.0 & 72.5 & 75.1 & \textbf{22.4} & 10.5 & 61.6 & \textbf{70.1} & 59.2 \\
      \rowcolor{white}
      \methodname & \textbf{73.5} & 46.2 & 80.3 & \textbf{80.7} & \textbf{73.1} & 76.8 & 22.2 & \textbf{10.7} & \textbf{63.0} & 67.7 & \textbf{59.4} \\
      \bottomrule
    \end{tabular}
    }
  \end{table*}

\paragraph{Long-Context Performance.}
Table~\ref{tab:results-ruler} (RULER) and Table~\ref{tab:results-helmet} (HELMET) show that~\methodname consistently outperforms NSA and InfLLMv2 in terms of both performance and sparsity. In particular, results on RULER and the Recall task in HELMET demonstrate that~\methodname exhibits a significantly stronger capability in challenging long-context retrieval tasks. Moreover, \methodname achieves better performance than the baselines across multiple model sizes and evaluation tasks. 
We further show in Table~\ref{tab:sp-tr-ds-inf} that~\methodname can also inference with softmax full attention.
\looseness=-1

\begin{wraptable}[11]{r}{0.4\textwidth}
  \centering
  \small
  \setlength{\tabcolsep}{6pt} 
  \vspace{2pt}
  \caption{
    Long-context performance of \methodname (DA) models with softmax full attention (FA).
  }
  \label{tab:sp-tr-ds-inf}
  \resizebox{0.4\textwidth}{!}{
  \begin{tabular}{lccccc}
  \toprule
  & \multicolumn{2}{c}{RULER} & \multicolumn{2}{c}{HELMET}  \\
  & FA & DA+FA & FA & DA+FA \\
  \midrule
\texttt{1B} & 66.6 & \textbf{70.4} & 32.5 & \textbf{33.8} \\
\texttt{3B} & 69.2 & \textbf{71.9} & \textbf{37.4} & 37.0 \\
\texttt{8B} & 85.3 & \textbf{86.7} & 47.7 & \textbf{48.7} \\

  \bottomrule \\
  \end{tabular}
  }
\end{wraptable}

\paragraph{\methodname with Softmax Decoding.} 
To close the gap with standard full softmax attention, we first investigate the impact of the chunk size in Stage 0 in App. \ref{sec:supp-chunk-sizes} (Table~\ref{tab:chunk-size-full}), where we find that reducing chunk size naturally leads to better results. However, this also reduces efficiency, as if chunk size is 1 (no chunking), \methodname reverts to \entmax attention followed by softmax attention, and thus it incurs the total cost of both at inference time. 
An alternative is to simply use softmax at test time without Stage~0 and Stage~1, where highly optimized implementations already exist (e.g., in vLLM~\citep{kwon2023efficient} and SGLang~\citep{zheng2024sglang}).
Table~\ref{tab:sp-tr-ds-inf} reports the results of training with DA and inference with FA, showing even higher performance than FA-trained models.

\paragraph{General Task Performance.}

We further evaluate~\methodname and the baselines on short-context general tasks to verify that~\methodname does not notably degrade short-context performance. During inference, both~\methodname and InfLLMv2 use softmax full attention, while NSA uses compressed attention without sparsity. As shown in Table~\ref{tab:general-tasks}, \methodname achieves performance on par with the original Full Attention model, and slightly outperforms NSA and InfLLMv2.

\subsection{Efficiency Benchmark}
\label{subsec:efficiency_bench}

We benchmark \methodname's two-stage kernels in two regimes: \emph{prefill} and \emph{decoding}. We compare against \textbf{FlashAttention3}~\citep{shah2024flashattention}, \textbf{NSA}~\citep{yuan-etal-2025-native},\footnote{No official NSA implementation has been released; we thus use a open-source Triton implementation at \url{https://github.com/XunhaoLai/native-sparse-attention-triton}.} and \textbf{InfLLMv2}~\citep{zhao2026infllmv}. All methods use a chunk size of $64$. We sweep chunk sparsity $s \in \{75\%, 87.5\%, 93.75\%\}$, with corresponding routing top-$k$ equal to $(1 - s) \cdot N / 64$; the \methodname active-block bitmask is randomized to the target sparsity to isolate the last stage's cost. Prefill is run at batch size $1$ and decoding at batch size $24$. See App.~\ref{sec:efficiency-appendix} for the full benchmark setup, including hardware and motivation for these batch sizes.

\paragraph{Prefill.}
\methodname is the fastest method at every operating point (Table~\ref{tab:hasa-efficiency}, left), with speedups over dense FlashAttention ranging from $1.34\times$ to $3.09\times$. The gap over InfLLMv2 and NSA is widest at the densest setting, where these methods' top-$k$ overhead is least amortized, and narrows toward extreme sparsity, but \methodname stays ahead in every cell.

\paragraph{Decoding.}
In the memory-bound decoding regime (Table~\ref{tab:hasa-efficiency}, right), \methodname is again the fastest at every operating point, reaching $3.36\times$ over FlashAttention at $n_{kv}{=}96$K, $s{=}93.75\%$, against $3.10\times$ for InfLLMv2. Compared to FA, the advantage grows monotonically with both context length and sparsity as the cost of dense attention eclipses Stage~1's routing overhead. Compared to InfLLMv2, this result follows from \methodname's Stage~2 (Sec.~\ref{sec:gpu}): we walk the bit-packed active-block mask in a single fused pass, avoiding the explicit top-$k$ and per-query index materialization InfLLMv2 performs between its scoring and attention stages.

\begin{table}[t]
  \centering
  \caption{Wall-clock speedup over full attention with FlashAttention (higher is better). Columns are chunk sparsity $s$, the fraction of chunks not routed to Stage~2; the corresponding top-$k$ is $(1 - s) \cdot N / 64$. Iv2 denotes InfLLMv2 and DA denotes \methodname. Best results in bold.}
  \label{tab:hasa-efficiency}
  \small
  \setlength{\tabcolsep}{5pt}
  \resizebox{\textwidth}{!}{%
  \begin{tabular}{l ccc ccc ccc @{\hspace{1.4em}} ccc ccc ccc}
    \toprule
        & \multicolumn{9}{c}{\textsc{Prefill}}
        & \multicolumn{9}{c}{\textsc{Decoding}} \\
    \cmidrule(lr){2-10} \cmidrule(lr){11-19}
        & \multicolumn{3}{c}{$s=75\%$}
        & \multicolumn{3}{c}{$s=87.5\%$}
        & \multicolumn{3}{c}{$s=93.7\%$}
        & \multicolumn{3}{c}{$s=75\%$}
        & \multicolumn{3}{c}{$s=87.5\%$}
        & \multicolumn{3}{c}{$s=93.7\%$} \\
    \cmidrule(lr){2-4}\cmidrule(lr){5-7}\cmidrule(lr){8-10}
    \cmidrule(lr){11-13}\cmidrule(lr){14-16}\cmidrule(lr){17-19}
    $N$ 
        & NSA & Iv2 & DA
        & NSA & Iv2 & DA
        & NSA & Iv2 & DA
        & NSA & Iv2 & DA
        & NSA & Iv2 & DA
        & NSA & Iv2 & DA \\
    \midrule
    $16$K
        & 0.71 & 0.96 & \textbf{1.34}
        & 1.23 & 1.63 & \textbf{2.09}
        & 1.91 & 2.40 & \textbf{2.87}
        & 0.40 & 1.20 & \textbf{1.54}
        & 0.43 & 1.43 & \textbf{1.67}
        & 0.46 & 1.73 & \textbf{1.93} \\
    $32$K
        & 0.76 & 1.04 & \textbf{1.39}
        & 1.25 & 1.67 & \textbf{2.11}
        & 2.19 & 2.78 & \textbf{2.94}
        & 0.57 & 1.46 & \textbf{1.69}
        & 0.65 & 1.85 & \textbf{2.26}
        & 0.73 & 2.29 & \textbf{2.69} \\
    $64$K
        & 0.82 & 1.01 & \textbf{1.55}
        & 1.33 & 1.77 & \textbf{2.28}
        & 2.24 & 2.86 & \textbf{3.02}
        & 0.71 & 1.59 & \textbf{1.97}
        & 0.95 & 2.28 & \textbf{2.69}
        & 1.12 & 2.92 & \textbf{3.27} \\
    $96$K
        & 0.82 & 1.04 & \textbf{1.63}
        & 1.31 & 1.73 & \textbf{2.34}
        & 2.32 & 3.06 & \textbf{3.09}
        & 0.80 & 1.73 & \textbf{1.96}
        & 1.04 & 2.38 & \textbf{2.72}
        & 1.34 & 3.10 & \textbf{3.36} \\
    \bottomrule
  \end{tabular}%
  }
\end{table}

\subsection{Cost-Effectiveness Analysis}
\label{subsec:pareto-frontiers}

We combine performance and efficiency to examine each method's accuracy--sparsity trade-off. \pgfplotsset{compat=1.18}

\definecolor{adansaColor}{HTML}{E67373}
\definecolor{infllmColor}{HTML}{64B5F6}
\definecolor{nsaColor}{HTML}{81C784}
\definecolor{fullAttnColor}{HTML}{555555}

\begin{wrapfigure}{r}{0.4\linewidth}
    \centering
    \vspace{-1em}

\begin{tikzpicture}
\begin{axis}[
    width=1\linewidth,
    height=0.7\linewidth,
    xmin=56,
    xmax=92,
    enlarge x limits=false,
    ymin=15,
    ymax=50,
    xtick={60,70,80,90},
    ytick={20,30,40,50},
    xlabel={Sparsity $\to$ (\%)},
    xlabel style={yshift=0.4em},
    ylabel={Overall Accuracy $\to$ (\%)},
    label style={font=\scriptsize},
    ylabel style={yshift=-0.4em},
    tick label style={font=\scriptsize},
    grid=major,
    grid style={gray!18},
    axis line style={black!70},
    tick style={black!70},
    major tick length=2pt,
    minor tick length=1pt,
    legend style={
        font=\scriptsize,
        draw=none,
        fill=none,
        at={(0.0,0.0)},
        anchor=south west,
        legend columns=1,
        row sep=-2pt,
        inner ysep=0pt,
        cells={anchor=west},
    },
    every axis plot/.append style={
        line width=1.4pt,
        mark=*,
        mark size=2.2pt,
        mark options={solid, fill=white}
    },
]

\addplot[color=fullAttnColor, dashed, mark=none, line width=1.2pt] coordinates {
    (56,45.096842)
    (92,45.096842)
};
\addlegendentry{FullAttn}

\addplot[color=nsaColor, mark options={solid, fill=nsaColor}] coordinates {
    (60.156250,39.937834)
    (69.921875,35.415457)
    (80.078125,30.298355)
    (89.843750,20.066030)
};
\addlegendentry{NSA}

\addplot[color=infllmColor, mark options={solid, fill=infllmColor}] coordinates {
    (60.156250,46.056438)
    (69.921875,44.517904)
    (80.078125,41.352660)
    (89.843750,30.473025)
};
\addlegendentry{InfLLMv2}

\addplot[color=adansaColor, mark options={solid, fill=adansaColor}] coordinates {
    (58.526129,45.694452)
    (67.651314,45.339642)
    (81.542443,43.023605)
    (88.453271,39.433376)
};
\addlegendentry{\methodname}

\end{axis}
\end{tikzpicture}
    \vspace{-8pt}
    \caption{Accuracy--Sparsity Pareto frontiers on HELMET.}
    \label{fig:pareto}
    \vspace{-1em}
\end{wrapfigure}
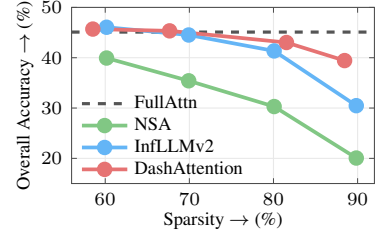

For the 8B model, we sweep the sparsity ($\gamma$ for \methodname\footnote{While $\alpha$ in \entmax also controls sparsity, we choose to adjust a temperature $\gamma$ instead as it allows us to re-use kernels optimized for specific values of $\alpha$, such as $\alpha=1.5$ in our case, which enjoys fast convergence and reduces all expressions to linear or quadratic operations; thus avoiding exponentials and logarithms and making the overall method faster.}, $k$ for NSA and InfLLMv2) of each sparse method to obtain points at increasing sparsity levels and report the resulting HELMET overall accuracy.
\methodname dominates NSA and InfLLMv2 across the sweep (Fig.~\ref{fig:pareto}) and slightly exceeds full attention at low to moderate sparsity.
As sparsity increases, the gap to the baselines widens sharply: at $\sim$90\% sparsity, \methodname retains 39.4\% overall accuracy, exceeding InfLLMv2 by 9\% and NSA by 19\%. 
This highlights \methodname's adaptiveness: fixed top-$k$ over-allocate easy queries or under-allocate hard ones, whereas \entmax reshapes the support adaptively.

\subsection{Dynamic Sparsity Analysis}
\label{subsec:ablation_studies}

\pgfplotsset{compat=1.18}

\definecolor{sparsityColor}{HTML}{E57373}

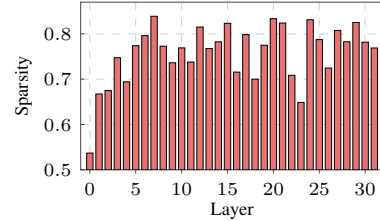
\begin{wrapfigure}[9]{r}{0.4\textwidth}
\centering
\vspace{-10pt}
\hspace*{-6pt}%
\begin{tikzpicture}
\begin{axis}[
    width=\linewidth,
    height=0.68\linewidth,
    ybar,
    bar width=2.4pt,
    xmin=-1,
    xmax=32,
    ymin=0.5,
    ymax=0.87,
    xtick={0,5,10,15,20,25,30},
    ytick={0.5,0.6,0.7,0.8},
    xtick pos=bottom,
    ytick pos=left,
    xlabel={Layer},
    ylabel={Sparsity},
    xlabel style={yshift=4pt},
    ylabel style={at={(axis description cs:-0.13,0.5)}, anchor=south},
    label style={font=\scriptsize},
    tick label style={font=\scriptsize},
    grid=major,
    grid style={gray!25, dashed},
    major grid style={draw=gray!35, dashed},
    axis line style={black!70},
    tick style={black!70, major tick length=1.4pt},
]

\addplot[
    fill=sparsityColor,
    draw=black,
    line width=0.35pt
] coordinates {
    (0,0.5370) (1,0.6674) (2,0.6749) (3,0.7474)
    (4,0.6942) (5,0.7739) (6,0.7961) (7,0.8389)
    (8,0.7728) (9,0.7363) (10,0.7691) (11,0.7378)
    (12,0.8151) (13,0.7676) (14,0.7824) (15,0.8233)
    (16,0.7156) (17,0.7984) (18,0.7001) (19,0.7750)
    (20,0.8335) (21,0.8240) (22,0.7087) (23,0.6487)
    (24,0.8312) (25,0.7874) (26,0.7246) (27,0.8077)
    (28,0.7826) (29,0.8252) (30,0.7817) (31,0.7689)
};

\end{axis}
\end{tikzpicture}
\vspace{-10pt}
\caption{Per-layer attention sparsity.}
\vspace{-10pt}
\label{fig:layer-sparsity}
\end{wrapfigure}

Since~\methodname uses \entmax in Stage 1 to induce sparsity, it can dynamically allocate sparsity across different layers according to the geometry of the Stage 1 scores. 
To show this property, we measure the sparsity of each layer using a 16K-length input from RULER-SG1. The analysis is shown in Fig.~\ref{fig:layer-sparsity}.
Notably, the early layers tend to remain denser, while the middle layers become sparser, automatically producing an effect similar to budget allocation strategies proposed in prior works~\citep{cai2024pyramidkv, yang2024pyramidinfer, lu2025moba}.

\section{Related Work}
\label{sec:related_works}

\paragraph{Attention and KV Cache Optimizations.}
Attention sparsification methods are developed to mitigate the computational and memory bottlenecks of long-context attention.
Early approaches use static patterns, such as attention sinks~\citep{xiao2024efficient} and sliding windows~\citep{Beltagy2020Longformer}, to approximate full attention, while later works introduce random~\citep{zaheer2020bigbird} and dynamic patterns~\citep{zhang2023h2o, jiang2024minference, kitaev2020reformer, zhang2025spargeattention}, e.g., block sparsity~\citep{jiang2024minference, tang2024quest}, to better exploit semantic structure.
Prior methods further exploit head heterogeneity by assigning different sparse patterns to different attention heads~\citep{jiang2024minference, ge2023model, xiao2024duoattention, lin2026lycheedecode}, improving long-context modeling.
Although these optimizations significantly improve efficiency, they introduce a training-inference mismatch.
In parallel, KV cache optimization methods have been proposed to reduce memory access cost and GPU memory usage.
Eviction methods~\citep{zhang2023h2o, li2024snapkv, cai2025r, bai2026indexcache} permanently discard redundant KV pairs, while KV offloading methods~\citep{xiao2024infllm, sun2024shadowkv, huang2025nosa} store most KV cache on CPU and fetch only query-relevant parts during decoding. KV cache compression can also be learned end-to-end through continued training \citep{huang2025nosa, dmc,lancucki2026inferencetime}.
KV quantization methods~\citep{hooper2024kvquant, liu2024kivi, zhang2025pqcache} further reduce memory footprint by storing KV cache in low-bit formats.
These approaches are orthogonal to our method, as they primarily compress KV cache rather than optimize attention sparsity.

\paragraph{Trainable Sparse Attentions.}
Recent approaches~\citep{li2025mtraining, gao2025seerattention} integrate sparse attention into LLM training to reduce training--inference mismatch and improve task performance.  
SeerAttn~\citep{gao2024seerattention}, NSA~\citep{yuan-etal-2025-native}, and MoBA~\citep{lu2025moba} introduce trainable block-sparse attention into LLMs, using top-$k$ selection over compressed attention scores to identify important blocks.  
InfLLMv2~\citep{zhao2026infllmv} improves the short-context inference efficiency of NSA by introducing a unified kernel in the final stage, while FSA~\citep{yan2025flash} further extends this design to smaller GQA group sizes.  
HSA~\citep{liuhsa} introduces a local encoder to summarize each chunk, but adds substantial parameters and is less compatible with pretrained softmax models during long-context continual training.  
Other token-wise sparse attention methods, such as DSA~\citep{deepseekai2025deepseekv32pushingfrontieropen} and CSA~\citep{deepseek2026v4}, are highly effective only at large model scales and under the MLA~\citep{liu2024deepseek} architecture.
The performance of these proposed methods are bounded by a fixed top-$k$ function, which does not introduce any dynamic property~\citep{zhang2026spargeattention2, ni2026double}.
Another line of work seeks to accelerate \entmax for modern hardware~\citep{ goncalves2025adasplash, goncalves2026adasplash2} .
However, they still require computing the full $\bm Q\bm K^\top$ tensor, drastically limiting inference acceleration.  
Our method \methodname bridges these two directions by introducing \entmax into hierarchical sparse attention while remaining easy to adapt from pretrained softmax models.

\section{Conclusion}
\label{sec:conclusion}

We propose \methodname, an end-to-end differentiable and adaptively sparse hierarchical attention mechanism. When integrated into long-context continual pretraining, \methodname outperforms existing sparse attention methods (NSA and InfLLMv2) in terms of both performance and speed. \methodname demonstrates a favorable cost-effectiveness trade-off in high-sparsity regimes and achieves speedups of 3.36$\times$ and 1.35$\times$ over FlashAttention-3 and InfLLMv2, respectively. One limitation of this work is that the \methodname kernels have not yet been integrated into modern LLM serving frameworks, such as vLLM and SGLang, which is left for future work. We will also explore applying \methodname to other model architectures, such as hybrid models~\citep{blakeman2025nemotron}.

\section*{Acknowledgments}

We would like to thank the SARDINE lab team for the helpful discussions.
This work is supported by the National Natural Science Foundation of China (No. 62576186), and Tsinghua University Initiative Scientific Research Program.
This work was supported  by the project DECOLLAGE (ERC-2022-CoG 101088763), by the Portuguese Recovery and Resilience Plan through project C645008882-00000055 (Center for Responsible AI), and by FCT/MECI through national funds and, when applicable, co-funded EU funds under DOI:10.54499/UID/50008/2025: Instituto de Telecomunicações. 
Edoardo M. Ponti is supported by the ERC Starting Grant AToM-FM (ERC-2025-StG 101222956).

\bibliographystyle{unsrt}
\bibliography{neurips_2026}

\appendix
\clearpage
\newpage

\begin{figure}[t]
    \centering
    \includegraphics[width=1\linewidth]{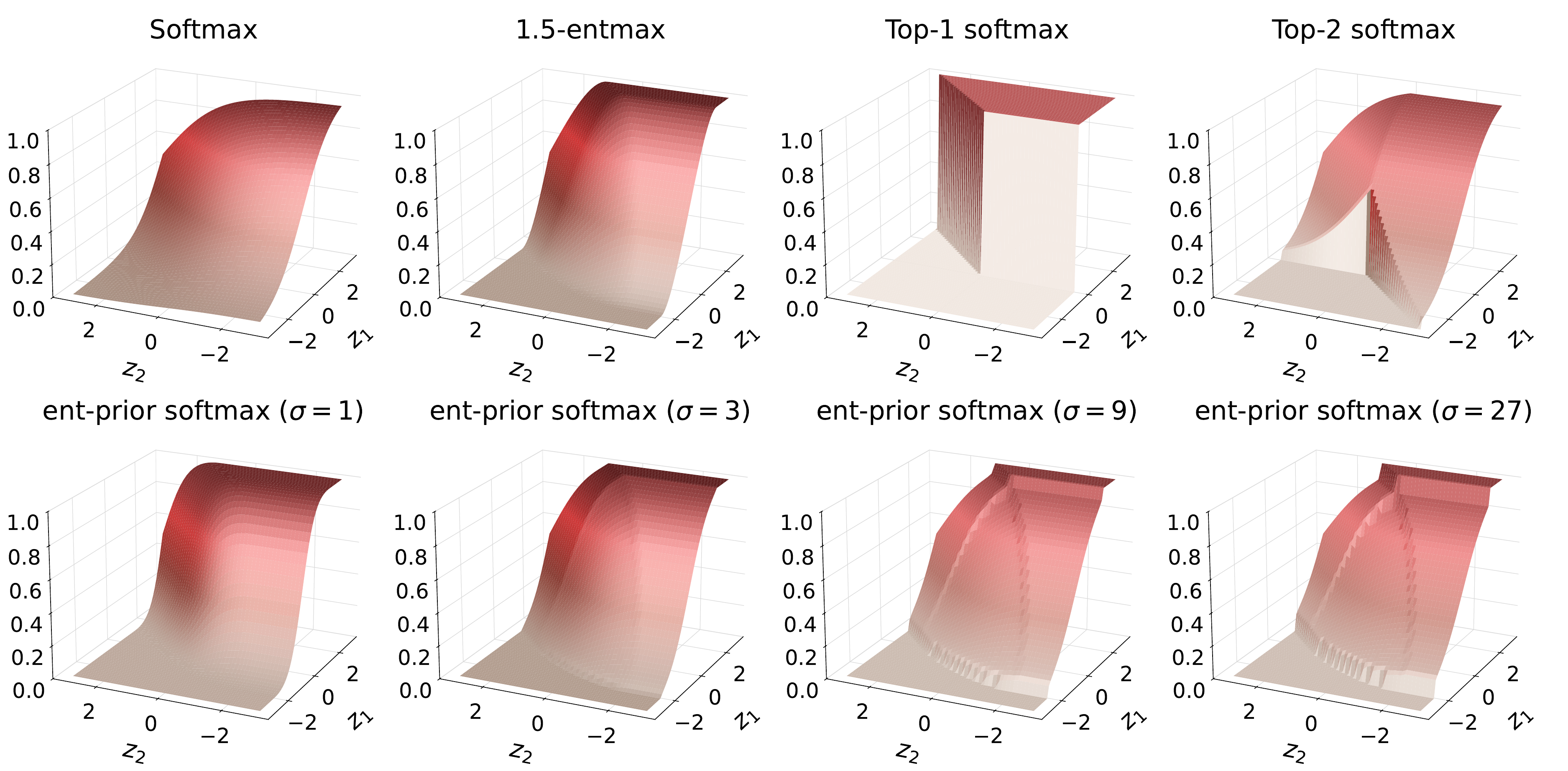}
    \caption{Visualization of mappings for different values of $\alpha$, and top-$k$ softmax with $k=1, 2$.
    Each panel shows how $p_0$ varies for the input $\bm{z} = [0, z_1, z_2]$.
    Ent-prior softmax denotes a variant of softmax in which the softmax scores are scaled by an entmax prior. The prior is also generated from $\bm z$ and normalized according to Proposition~\ref{th:attn-bias}.
    }
    \label{fig:entmax_visualizations_app}
\end{figure}

\section{The $\alpha$-entmax Mapping}
\label{section:app_alpha_entmax}

This appendix briefly reviews the $\alpha$-entmax transformation used in our coarse routing stage (Stage~1), and includes basic properties and calculation of $\alpha$-entmax.

\paragraph{Variational View of Mappings to Probability Simplex.}
Given a vector of scores $\bm{z}\in\mathbb{R}^n$, $\alpha$-entmax maps $\bm{z}$ to a probability vector by solving a regularized prediction problem over the simplex~\citep{peters-etal-2019-sparse}:
\begin{equation}
\label{eq:app_entmax_def}
\entmax(\bm{z})
=
\argmax_{\bm{p}\in\triangle_n}
\bm{p}^{\top}\bm{z} + H_{\alpha}(\bm{p}),
\qquad
\triangle_n
=
\left\{
\bm{p}\in\mathbb{R}^{n}:
\bm{p}\geq \mathbf{0},
\ \mathbf{1}^{\top}\bm{p}=1
\right\},
\end{equation}
where $H_{\alpha}$ is the Tsallis $\alpha$-entropy~\citep{Tsallis1988}:
\begin{equation}
\label{eq:app_tsallis_entropy}
H_{\alpha}(\bm{p})
=
\begin{cases}
\displaystyle
\frac{1}{\alpha(\alpha-1)}
\sum_{j=1}^{n}
\left(p_j - p_j^{\alpha}\right),
& \alpha \neq 1, \\[0.8em]
\displaystyle
-\sum_{j=1}^{n} p_j \log p_j,
& \alpha = 1.
\end{cases}
\end{equation}
The parameter $\alpha$ controls the shape of the regularizer and, consequently, the sparsity of the resulting distribution.
The optimal solution has a thresholding form. 
There exists a scalar threshold $\tau$ such that
\begin{equation}
\label{eq:app_entmax_solution}
\entmax(\bm{z})_i
=
\left[
(\alpha-1)z_i - \tau
\right]_{+}^{\frac{1}{\alpha-1}},
\end{equation}
where $\tau$ is chosen so that the entries sum to one.
Therefore, any coordinate satisfying $(\alpha-1)z_i \leq \tau$ is assigned exactly zero mass.
This property is central to our use of $\alpha$-entmax: the router does not merely downweight irrelevant chunks, but can remove them from the routed set altogether.

\paragraph{Special Cases.}
The $\alpha$-entmax family contains softmax and sparsemax as limiting or special cases.
As $\alpha \to 1^{+}$, the Tsallis entropy converges to Shannon entropy, and Eq.~\eqref{eq:app_entmax_def} recovers the standard softmax transformation.
At $\alpha=2$, the mapping becomes sparsemax~\citep{martins2016softmax}, which is the Euclidean projection of the score vector onto the simplex, i.e., $\operatorname{sparsemax}(\bm z) = \argmin_{\bm p \in \simplex_n} \|\bm z - \bm p\|_2$.
Intermediate values, such as $\alpha=1.5$, provide a useful compromise: they preserve differentiability while allowing exact zeros in the output distribution~\citep{peters-etal-2019-sparse}.

\paragraph{Computing the threshold.}
Evaluating \entmax reduces to finding the scalar threshold $\tau$ in Eq.~\eqref{eq:app_entmax_solution}.
For $\alpha=2$ and $\alpha=1.5$, specialized algorithms are available~\citep{martins2016softmax,peters-etal-2019-sparse}. For general $\alpha$, $\tau$ can be obtained by root finding, such as bisection over a bounded interval~\citep{blondel2019learning}.
Recent GPU implementations further accelerate this step with safeguarded higher-order updates~\citep{goncalves2025adasplash,goncalves2026adasplash2}.

\paragraph{Why it is useful for routing.}
In our setting, $\alpha$-entmax is applied at the chunk-routing level, where exact zeros have a direct computational interpretation: chunks assigned zero probability are not passed to the expensive second-stage attention computation.
This differs from top-$k$ softmax in two important ways.
First, top-$k$ fixes the number of selected chunks in advance, so the routing budget is an external hyperparameter rather than a consequence of the score distribution.
By contrast, $\alpha$-entmax determines the active support from the scores themselves, allowing different queries and sequences to route to different numbers of chunks.
Second, top-$k$ changes support through a discrete selection operator; within a region where the selected set is fixed, gradients do not describe how an unselected chunk should change its score to enter the routed set.
In $\alpha$-entmax, the routing probabilities arise from a continuous sparse transformation, so score changes near the boundary are reflected in the probability map rather than handled by a separate hard truncation step. Figure~\ref{fig:entmax_visualizations_app} illustrates that $\alpha$-entmax functions are smooth but cannot accelerate inference, since the full $\bm Q\bm K^\top$ matrix must still be computed. In contrast, top-$k$ softmax can accelerate inference but is non-smooth and stiff. With a suitable choice of $\sigma$, ent-prior softmax, similar to \methodname, can be smoother and more dynamic than top-$k$ softmax while still enabling inference acceleration through hierarchical attention.

\paragraph{Differentiability of \methodname.}
Entmax priors provide a native mechanism for enabling differentiability in hierarchical sparse attention. Unlike NSA, which relies on compressed attention to pass gradients to the first stage, or InfLLMv2, where the first stage is not trainable, our approach allows gradients to flow back through the routed chunks via 1.5-entmax. 
While 1.5-entmax can produce exact zeros, it admits a closed-form Jacobian within regions where the support remains fixed. At support-change boundaries, it is piecewise differentiable, analogous to sparsemax or ReLU. In practice, automatic differentiation applies the corresponding subgradient or Jacobian, enabling end-to-end gradient-based training. Therefore, from a computational perspective, DashAttention inherits the differentiability of $\alpha$-entmax and thus can be viewed as a fully end-to-end differentiable mechanism.

\section{Reproducibility Settings}
\label{sec:reproducibility}

\subsection{Performance Evaluation}
\label{sec:long-context-adaption-appendix}

\subsubsection{Training}
\label{sec:long-context-adaption-training-appendix}

We choose the 1B, 3B, and 8B base versions of MiniCPM-4~\citep{team2025minicpm4} as our backbone models, motivated by their relatively large GQA group sizes. Specifically, the 1B model uses a group size of 8, while the 3B and 8B models adopt a group size of 16.
We train all three models on InfLLM-5B~\citep{openbmb_InfLLM_V2_data_5B_2025} with a context length of 16K, followed by a short supervised fine-tuning (SFT) stage using the original MiniCPM-4~\citep{team2025minicpm4} training data. We employ the WSD learning rate scheduler~\citep{hu2024minicpm}, and conduct training using the Megatron framework~\citep{shoeybi2019megatron}.
All experiments are performed on 32 NVIDIA A800 GPUs with Intel® Xeon® Platinum 8470 CPUs running CentOS Linux 7 (Core), using \texttt{bfloat16} precision. For long-context continual pretraining, we train on 2.0B, 3.4B, and 5.0B tokens for the 1B, 3B, and 8B models, respectively. The subsequent SFT stage uses 0.6B, 1.3B, and 1.9B tokens correspondingly.
Notably, long-context continual pretraining is a standard practice for modern LLM development; therefore, our method introduces neither additional training stages nor extra training overhead compared with full attention and other baselines.

The~\methodname and baseline models are trained under the following settings. For NSA and InfLLMv2, we set the attention sink size to 64, the sliding window size to 1024, the block size to 64, and the total number of selected tokens to 4096. For~\methodname, the block size is also set to 64. We do not enforce attention to the sink tokens; instead, the sliding window consists of the diagonal block and the last block, resulting in a total size of 64--127 tokens. During training, we set $t_c = 1$ for the 1B and 3B models, and $t_c = 0.5$ for the 8B model. The parameter $\alpha$ is gradually increased from 1.25 to 1.5 throughout training. We set $\sigma=10^6$ for both 1B and 3B model, and $\sigma=10^8$ for the 8B model in both training and inference.

\subsubsection{Benchmarks}
\label{sec:long-context-adaption-benchmarks-appendix}

We evaluate~\methodname against all baselines on RULER~\citep{hsieh2024ruler} and HELMET~\citep{yen2025helmet}. We use the official RULER benchmark suite to assess long-context performance, and Olmes~\citep{gu2025olmes} for HELMET evaluation. For each RULER task, we evaluate 50 samples, and for each HELMET task, 100 samples, using greedy generation. All reported metrics and averaging methods follow the official settings. The reported sparsity is computed by excluding the first 4096 tokens, as these tokens use full attention in NSA and InfLLMv2. For general tasks, we evaluate the trained models on MMLU~\citep{wang2024mmlu}, MMLU-Pro~\citep{hendrycks2020measuring}, CSQA~\citep{talmor2019commonsenseqa}, IFEval~\citep{zhou2023instruction}, HellaSwag~\citep{zellers2019hellaswag}, GSM8K~\citep{cobbe2021training}, MATH~\citep{hendrycksmath2021}, DROP~\citep{dua2019drop}, MBPP~\citep{austin2021program}, and HumanEval~\citep{chen2021codex}, by using the LM-Evaluation-Harness~\citep{eval-harness} framework. For full attention, InfLLMv2, and~\methodname, we use softmax full attention during inference. For NSA, since removing the compressed attention during inference would cause training-inference mismatch, we inference the model with compressed attention scores but without sparsity in Stage~2.

\subsection{Efficiency Evaluation}
\label{sec:efficiency-appendix}

\paragraph{Setup.}
All measurements were taken on a single NVIDIA GH200 GPU with CUDA~13.2, using PyTorch~2.10.0 and Triton~3.6.0 at \texttt{bf16} precision. Each kernel is wrapped with \texttt{torch.compile} and timed with Triton's \texttt{do\_bench}; we measure the attention kernels themselves, with query/key/value projections and token-level KV-cache management out of scope. The GQA configuration is fixed across methods: $H_q{=}32$ query heads, $H_\mathrm{kv}{=}2$ key/value heads, head dimension $D{=}128$, and a chunk size of $64$ for all three sparse methods. We sweep $N \in \{16, 32, 64, 96\}$K and chunk sparsity $s \in \{75, 87.5, 93.75\}\%$, with routing top-$k = (1-s) \cdot N / 64$ for the top-$k$ baselines.

\paragraph{Batch sizes.}
We run prefill at batch size~$1$ and decoding at batch size~$24$. Prefill is set to batch size~$1$ to keep the Stage~1 chunk-score tensor within memory at $N{=}96$K. For decoding, batch size~$1$ latency is dominated by kernel-launch overhead, which CUDA graphs can mitigate, but we instead use $24$ to operate outside that regime.

\paragraph{Why we randomize the active-block bitmask.}
Stages~0 and~1 do not depend on the sparsity level $s$, so to sweep $s$ we inject at Stage~2 a random bit-packed mask at the target sparsity rather than a router-produced one. This isolates the cost of Stage~2 at a controlled sparsity and matches the effective sparsity used by the top-$k$ baselines.

\paragraph{Scope.}
These are isolated attention-kernel benchmarks (i.e., no request scheduler, no paged KV cache, etc.) so they do not translate one-for-one to end-to-end TTFT or token throughput in a production serving stack. Integration into a serving framework such as vLLM or SGLang is the natural next step and we leave it to future work.

\subsection{Performance and Efficiency Trade-Offs}
\label{sec:pareto-appendix}

We report the Pareto frontier at the in-distribution context length of 16K
in the main text. \methodname's effective sparsity is input-dependent and
controlled at inference by the scaling factor $t_c$; we sweep
$t_c \in \{0.50, 0.30, 0.15, 0.10\}$, with $t_c = 0.5$ recovering the
training operating point of the 8B checkpoint and the smaller values
probing denser regimes. NSA and InfLLMv2 instead use a deterministic
per-query top-$k$ budget, with average density
$\text{top-}k \times B / L$ ($B = 64$, $L = 16384$). Targeting densities
of $10\%$, $20\%$, $30\%$, and $40\%$ of the prefix to span roughly the
same range as~\methodname, we pick
$\text{top-}k \in \{26, 51, 77, 102\}$. Each point in the figure is
plotted at its empirically measured density.

\section{Extended Related Works}

In this section, we introduce the major related hierarchical trainable sparse attention and show the difference in the design between our method and existing methods.
Since we position our work as a hierarchical, block-wise trainable sparse attention method, we empirically compare it with the state-of-the-art approaches, NSA and InfLLMv2. For other lines of work, such as element-wise sparse attention, we discuss their design principles to highlight their orthogonality to key-novelties of our proposed method.

\paragraph{SeerAttention and MoBA.}
SeerAttention~\citep{gao2024seerattention} and MoBA~\citep{lu2025moba} are representative early works on trainable sparse attention. By compressing both Q and K vectors, these methods introduce dynamic block-sparse patterns through a top-$k$ selection function. Both methods are integrated into model training, allowing the model to adapt to sparse attention patterns and mitigating the train--inference mismatch that arises when attention sparsity is applied. However, because Q is also compressed within each chunk along the sequence dimension, these methods are difficult to apply during decoding. In contrast, \methodname selects KV chunks independently for each query, enabling exact alignment between training and inference as well as between prefill and decoding. This design improves long-context performance while accelerating both prefill and decoding.

\paragraph{NSA and InfLLMv2.}
NSA~\citep{yuan-etal-2025-native} first introduces a separate KV selection mechanism by leveraging the group size in GQA to satisfy the Tensor Core GEMM shape constraints. It adopts three complementary attention patterns: sliding-window attention, block-sparse attention, and compressed attention. The compressed KV representations are produced by an MLP with additional parameters introduced during long-context continual pretraining. However, since compressed attention is integrated into the model architecture, it is difficult to remove even when performing inference on short-context inputs, which can slow down short-context inference.
To address this issue, InfLLMv2~\citep{zhao2026infllmv} removes the compressed attention output from the final attention computation and replaces the MLP with mean pooling, thereby avoiding additional parameters. InfLLMv2 achieves both better performance and higher efficiency, and further provides an official efficient CUDA implementation. Nevertheless, both NSA and InfLLMv2 rely on top-$k$ selection, which lacks dynamic adaptability when selecting the essential KV cache entries to attend to.
\methodname alleviates this limitation by introducing an entmax-based Stage 1 mechanism. Moreover, since \methodname weakens the influence of prior scores through a large hyperparameter $\sigma$, it can also perform inference without Stage 1 scores and sparsity, aligning with the inference behavior of InfLLMv2. \methodname also removes the overlapped block compression strategy introduced in NSA and InfLLMv2, resulting in fewer chunk summarizations and accelerating Stage 1 without causing substantial performance degradation. In this paper, we choose these methods as baselines to demonstrate the effectiveness of the dynamicity introduced by entmax-based Stage 1.

\paragraph{Element-wise Sparse Attention.}
Adopted in DeepSeek-V3.2~\citep{deepseekai2025deepseekv32pushingfrontieropen} and DeepSeek-V4~\citep{deepseek2026v4}, element-wise hierarchical sparse attention methods, such as DSA~\citep{deepseekai2025deepseekv32pushingfrontieropen}, CSA, and HCA~\citep{deepseek2026v4}, have emerged as alternatives to block-selection methods by enabling finer-grained token selection. These methods typically train a token indexer to predict the importance of each KV position with respect to a given query input, supervised by an additional loss. Although they can achieve better performance than block-level sparse attention, their efficiency relies on large model scales and the use of the MLA architecture, which reduces the relative overhead introduced by the indexer. As a result, they are not directly compatible with the commonly adopted GQA architecture. In addition, these methods adopt a top-$k$ selection function.
Notably, \methodname is not fundamentally incompatible with element-wise sparse attention: a coarse-level token indexer could also be combined with entmax-based selection techniques and our prior-based training method. However, because training extremely large LLMs with element-wise sparse attention is currently infeasible for us, we leave the integration of \methodname with these methods to future work.

\paragraph{HSA.} HSA~\citep{hu2025hardware, hu2025every} has demonstrated the long-context performance and extrapolation capabilities. By using a local encoder and incorporating a \texttt{[CLS]} landmark token into the input sequence, HSA achieves strong long-context performance. However, the local encoder relies on a large number of newly introduced parameters and requires substantial training. Moreover, HSA cannot be directly adapted from existing pretrained softmax full-attention models, which limits its generalizability.
In contrast, \methodname's local attention mechanism degenerates to mean pooling at the beginning of training, providing a smooth transition from an existing full-attention model to a \methodname sparse-attention model. Since our comparison is conducted under a long-context continual pretraining setup, we exclude HSA as a baseline because it cannot be adapted from pretrained models.

\section{Equivalent Attention Bias Form of Diagonal Estimation}
\label{sec:equiv-bias-form}

We start by re-stating the original proposition.

\begin{restatetheorem}{Proposition~\ref{th:attn-bias}}
Calculating the attention output $\bm o_i$ with Eq.~\eqref{eq:stage2} and Eq.~\eqref{eq:lambda_factorization} is equivalent to first computing $\mu_i = \operatorname{mean}_{j\in \mathcal{R}_i}\{\log w_{i, c(j)}\}$, followed by
$$
d_{i,j} =
\begin{cases}
\dfrac{\log w_{i,c(j)} - \mu_i}{\sigma} \,, & j \in \mathcal{R}_i, \\[0.4em]
0\,, & j \in \mathcal{D}_i, \\
\end{cases}\quad
\bm o_i = \sum_{j\in \mathcal{R}_i\cup \mathcal{D}_i}
   \frac{\exp(z_{i,j} + d_{i,j})\,\bm v_j}
        {\sum_{t\in \mathcal{R}_i\cup \mathcal{D}_i} \exp(z_{i,t} + d_{i,t})}.
$$
\end{restatetheorem}

\begin{proof}
We start the proof by expanding $\lambda_i$:
\begin{equation*}
\lambda_i = \operatorname{sigmoid}\left(\operatorname{KL}(\bm u_{\mathcal{R}_i}\|\bm w'_{\mathcal{R}_i}) + \log\frac{|\mathcal{R}_i|}{|\mathcal{D}_i|}\right) = \frac{1}{1 + \exp\left(-\operatorname{KL}(\bm u_{\mathcal{R}_i} \|\bm w_{\mathcal{R}_i}') - \log\frac{|\mathcal{R}_i|}{|\mathcal{D}_i|}\right)}.
\end{equation*}
Since $\bm u_{\mathcal{R}_i}$ is uniform on $\mathcal{R}_i$, the KL term simplifies as
\begin{equation*}
\operatorname{KL}(\bm u_{\mathcal{R}_i} \| \bm w'_{\mathcal{R}_i})
  = \sum_{j\in \mathcal{R}_i} \frac{1}{|\mathcal{R}_i|}\log\frac{1/|\mathcal{R}_i|}{w'_{i,j}}
  = -\log|\mathcal{R}_i| - \frac{1}{|\mathcal{R}_i|}\sum_{j\in \mathcal{R}_i} \log w'_{i,j}.
\end{equation*}
Denote the geometric mean of $\bm w'_i$ on the routed support by
\begin{equation*}
\eta_i = \exp\!\left(\tfrac{1}{|\mathcal{R}_i|}\textstyle\sum_{j\in \mathcal{R}_i}\log w'_{i,j}\right) = \left(\textstyle\prod_{j\in\mathcal{R}_i} w'_{i,j}\right)^{1/|\mathcal{R}_i|}.
\end{equation*}
Plugging back, we have
\begin{equation*}
\lambda_i = \frac{1}{1 + |\mathcal{D}_i|\,\eta_i},\qquad 1 - \lambda_i = \frac{|\mathcal{D}_i|\,\eta_i}{1 + |\mathcal{D}_i|\,\eta_i}.
\end{equation*}
Therefore, for all $j\in \mathcal{R}_i$, we have
\begin{equation*}
\log g_{i,j} = \log \lambda_i + \log w'_{i,j} = \log w'_{i,j} - \log\!\left(1 + |\mathcal{D}_i|\,\eta_i\right),
\end{equation*}
and for all $j\in \mathcal{D}_i$, we have
\begin{equation*}
\log g_{i,j} = \log\!\left(\tfrac{1-\lambda_i}{|\mathcal{D}_i|}\right) = \log\eta_i - \log\!\left(1 + |\mathcal{D}_i|\,\eta_i\right).
\end{equation*}
Denote
\begin{equation*}
A = \log\eta_i - \log\!\left(1 + |\mathcal{D}_i|\,\eta_i\right),
\end{equation*}
then $\forall j\in \mathcal{D}_i,~ \log g_{i,j} - A = 0 = d_{i,j}$, and
\begin{equation*}
\forall j\in \mathcal{R}_i,~ \log g_{i,j} - A = \log w'_{i,j} - \log\eta_i.
\end{equation*}
Recalling $w'_{i,j} = w_{i,c(j)}^{1/\sigma}/(B\,\bm 1^\top \bm w_i^{1/\sigma})$, the normalizer $B\,\bm 1^\top \bm w_i^{1/\sigma}$ does not depend on $j$ and so cancels between $\log w'_{i,j}$ and $\log\eta_i$, giving
\begin{equation*}
\log g_{i,j} - A = \log w_{i,c(j)}^{1/\sigma} - \operatorname{mean}\{\log \bm w_i^{1/\sigma}\} = \frac{\log w_{i,c(j)} - \mu_i}{\sigma} = d_{i,j}.
\end{equation*}
Therefore, by plugging the results above into Eq.~\eqref{eq:stage2}, we can obtain that
\begin{equation*}
\begin{aligned}
\bm o_i &= \sum_{j\in \mathcal{R}_i\cup \mathcal{D}_i} \frac{g_{i,j}\exp(z_{i,j})\cdot \bm v_j}{\sum_{t\in \mathcal{R}_i\cup \mathcal{D}_i}g_{i,t}\exp(z_{i,t})} = \sum_{j\in \mathcal{R}_i\cup \mathcal{D}_i} \frac{\exp(z_{i,j} + \log g_{i,j} - A)\cdot \bm v_j}{\sum_{t\in \mathcal{R}_i\cup \mathcal{D}_i}\exp(z_{i,t} + \log g_{i,t} - A)} \\
&= \sum_{j\in \mathcal{R}_i\cup \mathcal{D}_i} \frac{\exp(z_{i,j} + d_{i,j})\cdot \bm v_j}{\sum_{t\in \mathcal{R}_i\cup \mathcal{D}_i}\exp(z_{i,t} + d_{i,t})}.
\end{aligned}
\end{equation*}
\end{proof}

\section{Non-Dispersion Property}
\label{sec:non-dispersion}

In this section, we focus on the dispersion property with respect to the Stage 1 design and the head aggregation strategy in hierarchical sparse attention. We begin by restating the definition of dispersion following~\citep{vasylenko2026longcontext}. Dispersion describes the phenomenon in which the entropy of attention probabilities grows at $\sim\log n$, introducing greater uncertainty as the sequence length increases and thereby making long-context modeling more challenging.

\begin{definition}[Dispersion]
    Given a bounded sequence $\bm z_n\in \mathbb{R}^n$ and a mapping $f: \mathbb{R}^n\to \triangle_{n}$, we call the mapping $f$ a dispersive function if 
    $$
        \lim_{n\to\infty} \frac{\mathcal{H}(f(\bm z_n))}{\log n} = 1,
    $$
    where $\mathcal{H}$ is the Shannon's entropy. If there exists a constant $0 \leq c < 1$, 
    $$
    \limsup_{n\to\infty} \frac{\mathcal{H}(f(\bm z_n))}{\log n}\leq c < 1,
    $$
    then $f$ is a non-dispersive function.
\end{definition}

Existing top-$k$ based hierarchical sparse attention methods, such as NSA~\citep{yuan-etal-2025-native} and InfLLMv2~\citep{zhao2026infllmv}, are non-dispersive in Stage 2. In particular, applying softmax only over the top-$k$ selected attention logits constitutes a non-dispersive operation.

\begin{theorem}[Top-$k$-Softmax is non-dispersive]
    $\operatorname{softmax}(\operatorname{top-}k(\cdot))$ is a non-dispersive function, i.e.,
    given a bounded sequence $\bm z\in \mathbb{R}^n$, there is always 
    $$\lim_{n\to\infty} \frac{\mathcal{H}(\operatorname{softmax}(\operatorname{top-}k(\bm z)))}{\log n} = 0.$$
\end{theorem}
\begin{proof}
    $\mathcal{H}(\operatorname{softmax}(\operatorname{top-}k(\bm z))) \leq \log k$, thus 
    $$\lim_{n\to\infty} \frac{\mathcal{H}(\operatorname{softmax}(\operatorname{top-}k(\bm z)))}{\log n} \leq \lim_{n\to\infty} \frac{\log k}{\log n} =  0.$$
\end{proof}

This shows that top-$k$ sparse attention mechanisms are non-dispersive and therefore can potentially perform well in long-context modeling.

However, when applying sparse attention to GQA models, existing methods commonly require all query heads within the same group to share an identical sparse pattern. To achieve this, a softmax-based head aggregation is typically applied before top-$k$ sparse attention, where the Stage 1 scores are merged on the probability simplex to prevent the aggregated scores from being overly dominated by a few strong heads. For \methodname, entmax head-aggregation is applied to achieve similar effect.

Next, we first formally define the head aggregation mechanism, followed by proving the dispersion property of softmax head aggregagation and non-dispersion property of entmax head disaggregation.

\begin{restatetheorem}[Head aggregation]
{Definition~\ref{de:head-aggregation}}
Given a mapping $f:\mathbb{R}^n\to \simplex_n$, $H$ bounded sequence $\{\bm z^{(h)}\}_{h=1}^H\subset \mathbb{R}^n$, and the aggregation weights $\theta = (\theta_1, \cdots, \theta_H)^\top\in \simplex_{H}$, the aggregated probability is 
$$
\operatorname{aggr}_f\left(\bm z^{(1)}, \bm z^{(2)}, \cdots, \bm z^{(H)};\bm \theta\right) = \sum_{h=1}^H \theta_h\cdot f\left(\bm z^{(h)}\right)\in \simplex_n.
$$
\end{restatetheorem}

\begin{restatetheorem}[Dispersion in head aggregation]
{Theorem~\ref{th:dispersion-head-aggregation}}
Under any finite $H$ and $\theta\in \simplex_{H}$, the following always hold:

1. Softmax head aggregation is dispersive, i.e. 
$$
\lim_{n\to\infty}\frac{\mathcal{H}\left(\operatorname{aggr}_{\operatorname{softmax}}\left(\bm z^{(1)}, \bm z^{(2)}, \cdots, \bm z^{(H)};\bm \theta\right)\right)}{\log n} = 1,
$$

2. Denote $\bm p^{(h)} = \alpha\operatorname{-entmax}(\bm z^{(h)})$, if there are $\|\bm p^{(h)}\|_0 = \mathcal{O}\left(n^{\beta_h}\right), \beta_h\in (0, 1)$, then entmax head aggregation is not dispersive, and 
$$
\limsup_{n\to\infty}\frac{\mathcal{H}\left(\operatorname{aggr}_{\alpha\operatorname{-entmax}}\left(\bm z^{(1)}, \bm z^{(2)}, \cdots, \bm z^{(H)};\bm \theta\right)\right)}{\log n} \leq \max_{h\in [H]} \beta_h < 1 .
$$
\end{restatetheorem}

\begin{proof}
We first prove that softmax head aggregation is dispersive.

Denote $\bm p^{(h)} = \operatorname{softmax}\left(\bm z^{(h)}\right)$. 

First, we find the lower bound of the numerator by using $\mathcal{H}(\cdot)$'s concavity.
\begin{equation*}
\begin{aligned}
\mathcal{H}\left(\operatorname{aggr}_{\operatorname{softmax}}\left(\bm z^{(1)}, \bm z^{(2)}, \cdots, \bm z^{(H)};\bm \theta\right)\right) &= \mathcal{H}(\theta_1\bm p^{(1)} + \cdots + \theta_H\bm p^{(H)}) \\&\geq \theta_1\mathcal{H}\left(\bm p^{(1)}\right) + \cdots + \theta_H\mathcal{H}\left(\bm p^{(H)}\right).
\end{aligned}
\end{equation*}
Then, since $\operatorname{softmax}$ is a dispersive function, we have 
\begin{equation*}
\begin{aligned}
\lim_{n\to\infty}\frac{\mathcal{H}\left(\operatorname{aggr}_{\operatorname{softmax}}\left(\bm z^{(1)}, \cdots, \bm z^{(H)};\bm \theta\right)\right)}{\log n} &\geq \theta_1\lim_{n\to\infty}\frac{\mathcal{H}\left(\bm p^{(1)}\right)}{\log_n} + \cdots + \theta_H\lim_{n\to\infty}\frac{\mathcal{H}\left(\bm p^{(H)}\right)}{\log_n} \\&= \theta_1 + \cdots + \theta_H = 1.
\end{aligned}
\end{equation*}

Then, since $\forall \bm p\in \triangle_{n}$, $\mathcal{H}(\bm p)\leq \log n$, we also have $\lim_{n\to\infty}\frac{\mathcal{H}(\bm p)}{\log n} \leq 1$. Therefore, we can obtain the following limitation:
$$
    \lim_{n\to\infty}  \frac{\mathcal{H}\left(\operatorname{softaggr}\left(\bm z^{(1)}, \bm z^{(2)}, \cdots, \bm z^{(H)};\bm \theta\right)\right)}{\log n} = 1.
$$

Then, we prove that the entmax head aggregation under given assumption is not dispersive.

Since $\forall h\in \{1, 2, \cdots, H\}, \|\bm p^{(h)}\|_0 = \mathcal{O}(n^{\beta_h})$, then 
$$\|\bm p\|_0 = \left\|\frac{1}{H}\sum_{h=1}^H \bm p^{(h)}\right\|_0
= \left\|\sum_{h=1}^H \bm p^{(h)}\right\|_0
\leq \sum_{h=1}^H \left\|\bm p^{(h)}\right\|_0
= \sum_{h=1}^H\mathcal{O}\left( n^{\beta_h}\right) = \mathcal{O}\left(n^{\max_{h\in[H]}\beta_h}\right),$$
therefore we can obtain that $\|\bm p\|_0 = \mathcal{O}\left(n^{\max_{h\in[H]}\beta_h}\right)$.

By using~\cite{vasylenko2026longcontext}'s Proposition 1, the theorem can be simply proved.

\end{proof}

\section{Additional Experimental Results}

\subsection{Full Benchmark Results of HELMET}
\label{sec:supp-helmet}
\begin{table*}[h]
    \small
        
    \centering
    \setlength{\tabcolsep}{2pt} 
    \caption{\textbf{\textit{HELMET}} scores (\%) of \methodname~compared with all baselines for sparse adaptation (detailed).}
    \label{tab:results-helmet-full}
    \resizebox{\textwidth}{!}{
    \begin{tabular}{ll>{\columncolor{black!5}}cccc>{\columncolor{black!5}}cccc>{\columncolor{black!5}}cccc}
    \toprule
    
    \multicolumn{2}{c}{\multirow{2}{*}{Task}} & \multicolumn{4}{c}{\texttt{MiniCPM-4-1B}} & \multicolumn{4}{c}{\texttt{MiniCPM-4-3B}} & \multicolumn{4}{c}{\texttt{MiniCPM-4-8B}} \\
    \cmidrule(lr){3-6} \cmidrule(lr){7-10} \cmidrule(lr){11-14} 
    & & FullAttn & NSA & InfLLMv2 & \methodname & FullAttn & NSA & InfLLMv2 & \methodname & FullAttn & NSA & InfLLMv2 & \methodname \\
    \midrule
    \multirow{5}{*}{Recall} 
    & \textit{JKV} & 48.0 & 20.0 & 21.0 & \textbf{23.0} & 55.0 & 24.0 & 45.0 & \textbf{53.0} & 96.0 & 25.0 & \textbf{91.0} & 73.0 \\
    & \textit{MK2} & 82.0 & 16.0 & 48.0 & \textbf{61.0} & 93.0 & 20.0 & 51.0 & \textbf{77.0} & 99.0 & 28.0 & 73.0 & \textbf{98.0} \\
    & \textit{MK3} & 43.0 & 5.0 & 19.0 & \textbf{33.0} & 70.0 & 9.0 & 38.0 & \textbf{45.0} & 99.0 & 14.0 & 62.0 & \textbf{82.0} \\
    & \textit{MV} & 93.5 & 52.5 & 86.8 & \textbf{90.3} & 96.3 & 66.5 & \textbf{98.0} & 92.5 & 99.0 & 73.0 & 99.5 & \textbf{100.0} \\
    & Avg. & 66.6 & 23.4 & 43.7 & \textbf{51.8} & 78.6 & 29.9 & 58.0 & \textbf{66.9} & 98.3 & 35.0 & 81.4 & \textbf{88.3} \\
    \midrule
    \multirow{6}{*}{ICL} 
    & \textit{Ban} & 47.0 & 36.0 & \textbf{54.0} & 51.0 & 41.0 & \textbf{57.0} & 51.0 & 50.0 & 72.0 & 73.0 & \textbf{75.0} & 71.0 \\
    & \textit{Cli} & 64.0 & 36.0 & 60.0 & \textbf{64.0} & 45.0 & 46.0 & 46.0 & \textbf{63.0} & 64.0 & 79.0 & \textbf{83.0} & 73.0 \\
    & \textit{NLU} & 63.0 & 46.0 & \textbf{62.0} & 60.0 & 61.0 & 55.0 & 48.0 & \textbf{64.0} & 77.0 & 75.0 & 76.0 & \textbf{78.0} \\
    & \textit{TrC} & 80.0 & 78.0 & 84.0 & \textbf{88.0} & 54.0 & 44.0 & \textbf{54.0} & 53.0 & 53.0 & 52.0 & 60.0 & \textbf{77.0} \\
    & \textit{TrF} & 41.0 & 34.0 & 36.0 & \textbf{42.0} & 23.0 & 24.0 & 22.0 & \textbf{26.0} & 14.0 & 20.0 & \textbf{24.0} & 15.0 \\
    & Avg. & 59.0 & 46.0 & 59.2 & \textbf{61.0} & 44.8 & 45.2 & 44.2 & \textbf{51.2} & 56.0 & 59.8 & \textbf{63.6} & 62.8 \\
    \midrule
    Rerank & Avg. & 13.9 & 15.0 & 15.5 & \textbf{16.5} & 25.3 & 20.5 & \textbf{24.8} & 20.0 & 37.8 & 27.9 & 36.8 & \textbf{39.7} \\
    \midrule
    \multirow{5}{*}{RAG} 
    & \textit{Hot} & 20.0 & \textbf{19.0} & 17.0 & 14.0 & 28.0 & 30.0 & \textbf{35.0} & \textbf{35.0} & 33.0 & 28.0 & \textbf{36.0} & 33.0 \\
    & \textit{NQ} & 25.0 & 32.0 & 31.0 & \textbf{35.0} & 40.0 & \textbf{42.0} & 41.0 & 41.0 & 45.0 & 43.0 & 44.0 & \textbf{46.0} \\
    & \textit{Pop} & 32.0 & 32.0 & \textbf{38.0} & 35.0 & 46.0 & 37.0 & 43.0 & \textbf{46.0} & 45.0 & 38.0 & \textbf{48.0} & \textbf{48.0} \\
    & \textit{Tri} & 61.0 & 57.0 & 62.0 & \textbf{70.0} & 80.0 & 73.0 & \textbf{78.0} & 69.0 & 90.0 & 84.0 & 88.0 & \textbf{89.0} \\
    & Avg. & 34.5 & 35.0 & 37.0 & \textbf{38.5} & 48.5 & 45.5 & \textbf{49.3} & 47.8 & 53.3 & 48.3 & \textbf{54.0} & \textbf{54.0} \\
    \midrule
    \multirow{4}{*}{LongQA} 
    & \textit{InC} & 45.0 & 38.0 & \textbf{40.0} & 34.0 & 39.0 & 42.0 & \textbf{46.0} & 38.0 & 49.0 & 46.0 & \textbf{50.0} & \textbf{50.0} \\
    & \textit{InQ} & 14.3 & 13.9 & \textbf{15.2} & 13.3 & 20.5 & 18.2 & \textbf{22.1} & 14.6 & 27.2 & 25.1 & \textbf{26.5} & 26.0 \\
    & \textit{Nar} & 61.0 & 52.0 & 61.0 & \textbf{62.0} & 68.0 & 59.0 & \textbf{75.0} & 55.0 & 81.0 & 82.0 & \textbf{83.0} & 81.0 \\
    & Avg. & 40.1 & 34.6 & \textbf{38.7} & 36.4 & 42.5 & 39.7 & \textbf{47.7} & 35.9 & 52.4 & 51.0 & \textbf{53.2} & 52.3 \\
    \midrule
    \multirow{3}{*}{Cite} 
    & \textit{ASQ} & 17.4 & 12.8 & \textbf{14.7} & 14.7 & 23.2 & 14.1 & 10.8 & \textbf{19.6} & 36.0 & 24.0 & \textbf{32.3} & 32.1 \\
    & \textit{Qam} & 0.4 & 0.1 & \textbf{0.4} & 0.1 & 1.4 & 0.5 & \textbf{1.1} & 1.1 & 4.0 & 3.3 & \textbf{4.3} & 2.8 \\
    & Avg. & 8.9 & 6.4 & \textbf{7.6} & 7.4 & 12.3 & 7.3 & 6.0 & \textbf{10.3} & 20.0 & 13.7 & \textbf{18.3} & 17.4 \\
    \midrule
    \multirow{3}{*}{Summ.} 
    & \textit{InS} & 1.0 & 0.8 & \textbf{0.9} & 0.3 & 2.4 & 1.0 & \textbf{1.1} & 1.0 & 5.3 & 3.7 & \textbf{4.1} & 3.8 \\
    & \textit{Mul} & 7.4 & 8.1 & 8.9 & \textbf{13.2} & 16.9 & 13.3 & \textbf{18.5} & 15.6 & 27.6 & \textbf{25.6} & 23.3 & 23.5 \\
    & Avg. & 4.2 & 4.4 & 4.9 & \textbf{6.7} & 9.7 & 7.1 & \textbf{9.8} & 8.3 & 16.5 & \textbf{14.7} & 13.7 & 13.7 \\
    \midrule
    \multicolumn{2}{c}{Overall $\uparrow$ (\%)} & 32.5 & 23.6 & 29.5 & \textbf{31.2} & 37.4 & 27.9 & 34.2 & \textbf{34.3} & 47.7 & 35.8 & 45.9 & \textbf{46.9} \\
    \multicolumn{2}{c}{Sparsity $\uparrow$ (\%)} & 0.0 & 75.0 & 75.0 & \textbf{75.4} & 0.0 & 75.0 & 75.0 & \textbf{75.4} & 0.0 & 75.0 & 75.0 & \textbf{75.4} \\
    \bottomrule
    \end{tabular}
    }
    \end{table*}

Here, we list the detailed results of \methodname compared with baselines on HELMET in Table~\ref{tab:results-helmet-full}.

\subsection{Ablations on Chunk Sizes}
\label{sec:supp-chunk-sizes}

\begin{table*}[h]
    \small
    
    \centering
    \setlength{\tabcolsep}{2.5pt} 
    \caption{\methodname's performance under different chunk sizes (trained with 64).}
    \label{tab:chunk-size-full}
    \resizebox{\textwidth}{!}{
    \begin{tabular}{ccccccccccccccccc}
    \toprule
    \multirow{2}{*}{Chunk Size} & \multicolumn{13}{c}{RULER-16K} & \multirow{2}{*}{\makecell{Avg. $\uparrow$\\(\%)}} & \multirow{2}{*}{\makecell{Sparsity. $\uparrow$\\(\%)}} \\
    \cmidrule{2-14}
        & {SG1} & {SG2} & {SG3} & {MK1} & {MK2} & {MK3} & {MV} & {MQ} & {VT} & {CWE} & {FWE} & {QA1} & {QA2} &  &  \\
    \midrule
    16 & 100.0 & 100.0 & 100.0 & 100.0 & 98.0 & 100.0 & 100.0 & 99.0 & 57.2 & 50.6 & 76.0 & 70.0 & 50.0 & 84.7 & 77.0 \\
    32 & 100.0 & 100.0 & 100.0 & 98.0 & 96.0 & 86.0 & 100.0 & 99.0 & 60.0 & 49.6 & 72.0 & 72.0 & 54.0 & 83.6 & 75.7 \\
    64 & 100.0 & 100.0 & 96.0 & 96.0 & 92.0 & 68.0 & 96.0 & 95.0 & 54.0 & 45.0 & 66.7 & 70.0 & 54.0 & 79.4 & 75.1 \\
    128 & 100.0 & 92.0 & 98.0 & 90.0 & 80.0 & 58.0 & 83.5 & 88.5 & 50.0 & 48.8 & 59.3 & 74.0 & 46.0 & 74.5 & 75.0 \\
    
    \bottomrule
    \end{tabular}
    }
    \end{table*}

\methodname trained with one chunk size can be used for inference with different chunk sizes. A smaller chunk size enables finer-grained Stage~1 selection, leading to better performance under the same sparsity constraint. To verify this, we evaluate our 8B model, trained with a chunk size of 64, on RULER using test-time chunk sizes of $\{16, 32, 64, 128\}$. We adjust the entmax temperature to achieve  $\sim75\%$ sparsity. As shown in Table~\ref{tab:chunk-size-full}, smaller inference chunk sizes yield better task performance, demonstrating that the inference chunk size is not constrained by the chunk size used during training.

\subsection{Full Results of \methodname Inference in Full Attention}
\label{sec:supp-sp-tr-ds-inf}
\begin{table*}[h]
  \centering
  \small
  \setlength{\tabcolsep}{4.5pt}
  \caption{
    RULER performance of \methodname models inference with softmax full attention (detailed).
  }
  \label{tab:ruler-fullattn-adansa-full}
  \resizebox{\textwidth}{!}{
  \begin{tabular}{llcccccccccccccc}
  \toprule
  & \multirow{2}{*}{Method} & \multicolumn{14}{c}{RULER} \\
  \cmidrule(lr){3-16}
  & & SG1 & SG2 & SG3 & MK1 & MK2 & MK3 & MV & MQ & VT & CWE & FWE & QA1 & QA2 & Avg. \\
  \midrule
  \multirow{2}{*}{\texttt{1B}}
  & FullAttn & 100.0 & 100.0 & 100.0 & \textbf{98.0} & 78.0 & 48.0 & 87.0 & \textbf{89.0} & 17.6 & \textbf{1.8} & 66.7 & 44.0 & 36.0 & 66.6 \\
  & \methodname-full & \textbf{100.0} & \textbf{100.0} & \textbf{100.0} & 92.0 & \textbf{92.0} & \textbf{46.0} & \textbf{88.0} & 88.5 & \textbf{62.0} & 0.8 & \textbf{71.3} & 38.0 & \textbf{36.0} & \textbf{70.4} \\
  \midrule
  \multirow{2}{*}{\texttt{3B}}
  & FullAttn & \textbf{100.0} & \textbf{100.0} & 98.0 & \textbf{100.0} & 94.0 & \textbf{60.0} & \textbf{96.0} & \textbf{96.5} & 20.4 & \textbf{5.8} & 19.3 & \textbf{60.0} & \textbf{50.0} & 69.2 \\
  & \methodname-full & \textbf{100.0} & \textbf{100.0} & \textbf{100.0} & \textbf{100.0} & \textbf{96.0} & 56.0 & 89.0 & 94.5 & \textbf{21.6} & 2.8 & \textbf{82.7} & 50.0 & 42.0 & \textbf{71.9} \\
  \midrule
  \multirow{2}{*}{\texttt{8B}}
  & FullAttn & \textbf{100.0} & \textbf{100.0} & \textbf{100.0} & 98.0 & \textbf{100.0} & 96.0 & \textbf{100.0} & \textbf{99.5} & 58.4 & 49.2 & \textbf{78.0} & \textbf{78.0} & \textbf{52.0} & 85.3 \\
  & \methodname-full & \textbf{100.0} & \textbf{100.0} & \textbf{100.0} & \textbf{100.0} & \textbf{100.0} & \textbf{100.0} & 99.5 & 99.0 & \textbf{77.2} & \textbf{56.2} & 74.7 & 70.0 & 50.0 & \textbf{86.7} \\
  \bottomrule
  \end{tabular}
  }
\end{table*}

\begin{table*}[t]
  \centering
  \small
  \setlength{\tabcolsep}{9pt}
  \caption{
    HELMET performance of \methodname models inference with softmax full attention (detailed). 
  }
  \label{tab:helmet-fullattn-adansa-full}
  \resizebox{\textwidth}{!}{
  \begin{tabular}{llcccccccc}
  \toprule
  & \multirow{2}{*}{Method} & \multicolumn{8}{c}{HELMET} \\
  \cmidrule(lr){3-10}
  & & Recall & ICL & Rerank & RAG & LongQA & Cite & Summ. & Overall \\
  \midrule
  \multirow{2}{*}{\texttt{1B}}
  & FullAttn & \textbf{66.6} & 59.0 & 13.9 & 34.5 & \textbf{40.1} & \textbf{8.9} & 4.2 & 32.5 \\
  & \methodname-full & 66.2 & \textbf{61.2} & \textbf{17.5} & \textbf{39.8} & 37.5 & 8.6 & \textbf{5.7} & \textbf{33.8} \\
  \midrule
  \multirow{2}{*}{\texttt{3B}}
  & FullAttn & 78.6 & 44.8 & \textbf{25.3} & \textbf{48.5} & \textbf{42.5} & \textbf{12.3} & 9.7 & \textbf{37.4} \\
  & \methodname-full & \textbf{80.3} & \textbf{48.2} & 23.1 & 46.8 & 40.2 & 10.5 & \textbf{9.7} & 37.0 \\
  \midrule
  \multirow{2}{*}{\texttt{8B}}
  & FullAttn & 98.3 & 56.0 & 37.8 & \textbf{53.3} & 52.4 & 20.0 & \textbf{16.5} & 47.7 \\
  & \methodname-full & \textbf{99.8} & \textbf{59.0} & \textbf{40.5} & \textbf{53.3} & \textbf{53.7} & \textbf{20.9} & 13.6 & \textbf{48.7} \\
  \bottomrule
  \end{tabular}
  }
\end{table*}

Similar to InfLLMv2~\citep{zhao2026infllmv}, models trained with \methodname can perform inference with vanilla softmax full attention without notable performance degradation. As shown in Table~\ref{tab:ruler-fullattn-adansa-full} and Table~\ref{tab:helmet-fullattn-adansa-full}, \methodname-full outperforms FullAttn in most settings, demonstrating that models trained with sparse attention can even achieve better long-context performance than those trained with full attention when evaluated using full-attention inference.

\subsection{Cost-Effectiveness Analysis}

\pgfplotsset{compat=1.18}

\definecolor{adansaColor}{HTML}{E67373}
\definecolor{infllmColor}{HTML}{64B5F6}
\definecolor{nsaColor}{HTML}{81C784}
\definecolor{fullAttnColor}{HTML}{555555}

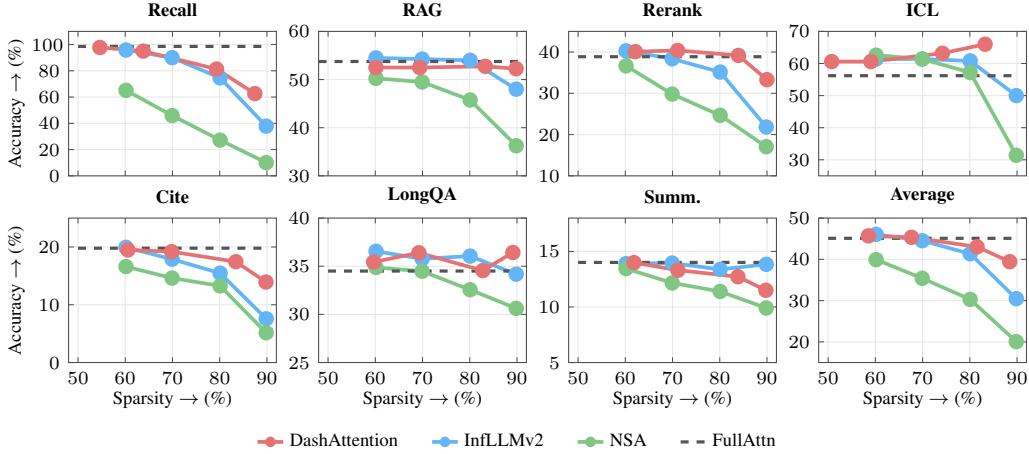
\begin{figure*}[h]
\centering
\vspace{-0.5em}
\begin{tikzpicture}
\begin{groupplot}[
    group style={
        group size=4 by 2,
        horizontal sep=0.04\linewidth,
        vertical sep=0.04\linewidth,
    },
    width=0.31\linewidth,
    height=0.25\linewidth,
    xmin=48,
    xmax=92,
    enlarge x limits=false,
    xtick={50,60,70,80,90},
    label style={font=\scriptsize},
    tick label style={font=\scriptsize},
    yticklabel style={text width=1em, align=right},
    title style={font=\scriptsize\bfseries, yshift=-0.4em},
    grid=major,
    grid style={gray!18},
    axis line style={black!70},
    tick style={black!70},
    major tick length=2pt,
    minor tick length=1pt,
    every axis plot/.append style={
        line width=1.4pt,
        mark=*,
        mark size=2.2pt,
        mark options={solid, fill=white}
    },
]

\nextgroupplot[
    title={Recall},
    ymin=0,
    ymax=110,
    ytick={0,20,40,60,80,100},
    xticklabel=\empty,
    ylabel={Accuracy $\to$ (\%)},
    ylabel style={yshift=-0.4em},
]
\addplot[color=infllmColor, mark options={solid, fill=infllmColor}, forget plot] coordinates {
    (60.156250,95.812500)
    (69.921875,90.000000)
    (80.078125,74.625000)
    (89.843750,37.875000)
};
\addplot[color=nsaColor, mark options={solid, fill=nsaColor}, forget plot] coordinates {
    (60.156250,65.125000)
    (69.921875,45.937500)
    (80.078125,27.250000)
    (89.843750,10.062500)
};
\addplot[color=fullAttnColor, dashed, mark=none, line width=1.2pt, forget plot] coordinates {
    (50,98.562500)
    (90.5,98.562500)
};
\addplot[color=adansaColor, mark options={solid, fill=adansaColor}, forget plot] coordinates {
    (54.628800,97.750000)
    (63.799400,95.000000)
    (79.329400,81.312500)
    (87.435900,62.625000)
};

\nextgroupplot[
    title={RAG},
    ymin=30,
    ymax=60,
    ytick={30,40,50,60},
    xticklabel=\empty,
]
\addplot[color=infllmColor, mark options={solid, fill=infllmColor}, forget plot] coordinates {
    (60.156250,54.500000)
    (69.921875,54.250000)
    (80.078125,54.000000)
    (89.843750,48.000000)
};
\addplot[color=nsaColor, mark options={solid, fill=nsaColor}, forget plot] coordinates {
    (60.156250,50.250000)
    (69.921875,49.500000)
    (80.078125,45.750000)
    (89.843750,36.250000)
};
\addplot[color=fullAttnColor, dashed, mark=none, line width=1.2pt, forget plot] coordinates {
    (48,53.750000)
    (92,53.750000)
};
\addplot[color=adansaColor, mark options={solid, fill=adansaColor}, forget plot] coordinates {
    (60.142500,52.500000)
    (69.417000,52.500000)
    (83.295900,52.750000)
    (89.822000,52.250000)
};

\nextgroupplot[
    title={Rerank},
    ymin=10,
    ymax=45,
    ytick={10,20,30,40},
    xticklabel=\empty,
]
\addplot[color=infllmColor, mark options={solid, fill=infllmColor}, forget plot] coordinates {
    (60.156250,40.291030)
    (69.921875,38.365960)
    (80.078125,35.142950)
    (89.843750,21.844610)
};
\addplot[color=nsaColor, mark options={solid, fill=nsaColor}, forget plot] coordinates {
    (60.156250,36.659670)
    (69.921875,29.807870)
    (80.078125,24.668740)
    (89.843750,17.067990)
};
\addplot[color=fullAttnColor, dashed, mark=none, line width=1.2pt, forget plot] coordinates {
    (50,38.867110)
    (90.5,38.867110)
};
\addplot[color=adansaColor, mark options={solid, fill=adansaColor}, forget plot] coordinates {
    (62.131500,40.091200)
    (71.087200,40.391380)
    (83.946900,39.178430)
    (89.982200,33.300170)
};

\nextgroupplot[
    title={ICL},
    ymin=25,
    ymax=70,
    ytick={30,40,50,60,70},
    xticklabel=\empty,
]
\addplot[color=infllmColor, mark options={solid, fill=infllmColor}, forget plot] coordinates {
    (60.156250,61.400000)
    (69.921875,61.400000)
    (80.078125,60.800000)
    (89.843750,50.000000)
};
\addplot[color=nsaColor, mark options={solid, fill=nsaColor}, forget plot] coordinates {
    (60.156250,62.600000)
    (69.921875,61.400000)
    (80.078125,57.200000)
    (89.843750,31.400000)
};
\addplot[color=fullAttnColor, dashed, mark=none, line width=1.2pt, forget plot] coordinates {
    (50,56.200000)
    (90.5,56.200000)
};
\addplot[color=adansaColor, mark options={solid, fill=adansaColor}, forget plot] coordinates {
    (50.755600,60.600000)
    (59.063800,60.600000)
    (74.198500,63.200000)
    (83.243700,66.000000)
};

\nextgroupplot[
    title={Cite},
    ymin=0,
    ymax=25,
    ytick={0,10,20},
    xlabel={Sparsity $\to$ (\%)},
    xlabel style={yshift=0.4em},
    ylabel={Accuracy $\to$ (\%)},
    ylabel style={yshift=-0.4em},
]
\addplot[color=infllmColor, mark options={solid, fill=infllmColor}, forget plot] coordinates {
    (60.156250,19.937876)
    (69.921875,17.885627)
    (80.078125,15.465998)
    (89.843750,7.577477)
};
\addplot[color=nsaColor, mark options={solid, fill=nsaColor}, forget plot] coordinates {
    (60.156250,16.600174)
    (69.921875,14.619760)
    (80.078125,13.262486)
    (89.843750,5.150647)
};
\addplot[color=fullAttnColor, dashed, mark=none, line width=1.2pt, forget plot] coordinates {
    (50,19.789781)
    (90.5,19.789781)
};
\addplot[color=adansaColor, mark options={solid, fill=adansaColor}, forget plot] coordinates {
    (60.550500,19.489357)
    (69.825500,19.181005)
    (83.420300,17.452093)
    (89.792800,13.930141)
};

\nextgroupplot[
    title={LongQA},
    ymin=25,
    ymax=40,
    ytick={25,30,35,40},
    xlabel={Sparsity $\to$ (\%)},
    xlabel style={yshift=0.4em},
]
\addplot[color=infllmColor, mark options={solid, fill=infllmColor}, forget plot] coordinates {
    (60.156250,36.551150)
    (69.921875,35.771378)
    (80.078125,36.050952)
    (89.843750,34.175516)
};
\addplot[color=nsaColor, mark options={solid, fill=nsaColor}, forget plot] coordinates {
    (60.156250,34.876259)
    (69.921875,34.479287)
    (80.078125,32.560445)
    (89.843750,30.622554)
};
\addplot[color=fullAttnColor, dashed, mark=none, line width=1.2pt, forget plot] coordinates {
    (50,34.499485)
    (90.5,34.499485)
};
\addplot[color=adansaColor, mark options={solid, fill=adansaColor}, forget plot] coordinates {
    (59.645000,35.439728)
    (69.231800,36.395050)
    (82.765200,34.547737)
    (89.133400,36.425611)
};

\nextgroupplot[
    title={Summ.},
    ymin=5,
    ymax=18,
    ytick={5,10,15},
    xlabel={Sparsity $\to$ (\%)},
    xlabel style={yshift=0.4em},
]
\addplot[color=infllmColor, mark options={solid, fill=infllmColor}, forget plot] coordinates {
    (60.156250,13.902512)
    (69.921875,13.952366)
    (80.078125,13.383723)
    (89.843750,13.838568)
};
\addplot[color=nsaColor, mark options={solid, fill=nsaColor}, forget plot] coordinates {
    (60.156250,13.453733)
    (69.921875,12.163781)
    (80.078125,11.396816)
    (89.843750,9.908518)
};
\addplot[color=fullAttnColor, dashed, mark=none, line width=1.2pt, forget plot] coordinates {
    (50,14.009016)
    (90.5,14.009016)
};
\addplot[color=adansaColor, mark options={solid, fill=adansaColor}, forget plot] coordinates {
    (61.829000,13.990880)
    (71.134500,13.310058)
    (83.840900,12.724478)
    (89.762900,11.502713)
};

\nextgroupplot[
    title={Average},
    ymin=15,
    ymax=50,
    ytick={20,30,40,50},
    xlabel={Sparsity $\to$ (\%)},
    xlabel style={yshift=0.4em},
]
\addplot[color=infllmColor, mark options={solid, fill=infllmColor}, forget plot] coordinates {
    (60.156250,46.056438)
    (69.921875,44.517904)
    (80.078125,41.352660)
    (89.843750,30.473025)
};
\addplot[color=nsaColor, mark options={solid, fill=nsaColor}, forget plot] coordinates {
    (60.156250,39.937834)
    (69.921875,35.415457)
    (80.078125,30.298355)
    (89.843750,20.066030)
};
\addplot[color=fullAttnColor, dashed, mark=none, line width=1.2pt, forget plot] coordinates {
    (50,45.096842)
    (90.5,45.096842)
};
\addplot[color=adansaColor, mark options={solid, fill=adansaColor}, forget plot] coordinates {
    (58.526129,45.694452)
    (67.651314,45.339642)
    (81.542443,43.023605)
    (88.453271,39.433376)
};

\end{groupplot}
\end{tikzpicture}
\vspace{-0.25em}
\makebox[\linewidth][c]{\scriptsize
    \tikz[baseline=-0.5ex]{\draw[adansaColor, line width=1.4pt] (0,0) -- (0.34,0) node[pos=0.5, circle, fill=adansaColor, inner sep=1.4pt] {};} \methodname
    \hspace{0.3cm}
    \tikz[baseline=-0.5ex]{\draw[infllmColor, line width=1.4pt] (0,0) -- (0.34,0) node[pos=0.5, circle, fill=infllmColor, inner sep=1.4pt] {};} InfLLMv2
    \hspace{0.3cm}
    \tikz[baseline=-0.5ex]{\draw[nsaColor, line width=1.4pt] (0,0) -- (0.34,0) node[pos=0.5, circle, fill=nsaColor, inner sep=1.4pt] {};} NSA
    \hspace{0.3cm}
    \tikz[baseline=-0.5ex]{\draw[fullAttnColor, dashed, line width=1.2pt] (0,0) -- (0.34,0);} FullAttn
}
\vspace{-0.5em}
\caption{Accuracy--Sparsity Pareto frontiers on HELMET at 16K context length.}
\label{fig:pareto_all}
\vspace{-1em}
\end{figure*}

We present detailed results for the cost-effectiveness analysis across the seven HELMET tasks in Fig.~\ref{fig:pareto_all}. 
The results show that \methodname is generally more robust under sparse settings than NSA and InfLLMv2, especially on tasks such as Recall, RAG, Rerank, and Cite. 
By using an $\alpha$-entmax-guided chunk routing strategy, \methodname allocates the sparsity budget more effectively across layers and heads under high sparsity, thereby minimizing performance degradation even in extreme configurations.

\subsection{Ablations on Local Attention and $\sigma$}
\label{sec:supp-local-attn-sigma}

\begin{table*}[h]
    \small
    
    \centering
    \setlength{\tabcolsep}{3.5pt} 
    \caption{Performance under different local attention and $\sigma$ settings.}
    \label{tab:local-attn-sigma-full}
    \resizebox{\textwidth}{!}{
    \begin{tabular}{clcccccccccccccc}
    \toprule
    \multirow{2}{*}{Local Attn} & \multicolumn{1}{c}{\multirow{2}{*}{$\sigma$}} & \multicolumn{13}{c}{RULER-16K} & \multirow{2}{*}{\makecell{Avg. $\uparrow$\\(\%)}}  \\
    \cmidrule{3-15}
        &  & {SG1} & {SG2} & {SG3} & {MK1} & {MK2} & {MK3} & {MV} & {MQ} & {VT} & {CWE} & {FWE} & {QA1} & {QA2} &  \\
    \midrule
    \ding{55} & $10^8$ & \textbf{100.0} & \textbf{100.0} & \textbf{100.0} & \textbf{100.0} & \textbf{96.0} & 84.0 & \textbf{100.0} & 98.0 & 59.6 & 47.4 & 71.3 & 70.0 & \textbf{56.0} & 83.3 \\
    \midrule
    \ding{51} & $10^{1}$ & \textbf{100.0} & \textbf{100.0} & \textbf{100.0} & 98.0 & 94.0 & 82.0 & 99.5 & 90.0 & \textbf{60.4} & 46.4 & 69.3 & \textbf{72.0} & 54.0 & 82.0 \\
    \ding{51} & $10^{4}$ & \textbf{100.0} & \textbf{100.0} & \textbf{100.0} & 96.0 & \textbf{96.0} & 82.0 & 99.5 & 98.5 & 60.0 & 48.4 & 71.3 & \textbf{72.0} & \textbf{56.0} & 83.1 \\
    \ding{51} & $10^{12}$ & \textbf{100.0} & \textbf{100.0} & \textbf{100.0} & \textbf{100.0} & \textbf{96.0} & 84.0 & 99.5 & \textbf{99.0} & 60.0 & 45.6 & \textbf{72.0} & 70.0 & 54.0 & 83.1 \\
    \midrule
    \ding{51} & $10^8$ & \textbf{100.0} & \textbf{100.0} & \textbf{100.0} & 98.0 & \textbf{96.0} & \textbf{86.0} & \textbf{100.0} & \textbf{99.0} & 60.0 & \textbf{49.6} & \textbf{72.0} & \textbf{72.0} & 54.0 & \textbf{83.6} \\
    
    \bottomrule
    \end{tabular}
    }
    \end{table*}

Here, we study the effectiveness of local attention and the role of $\sigma$ in controlling prior strength through ablations. To assess local attention, we replace it with mean pooling, following prior work~\citep{zhao2026infllmv, lu2025moba}. We also evaluate different prior scaling factors, $\sigma \in \{10^1, 10^4, 10^8, 10^{12}\}$, where larger $\sigma$ weakens the entmax priors. As $\sigma$ increases, the priors gradually diminish; in the limit $\sigma \to +\infty$, they degenerate into non-differentiable binary masks, making the summarization query head $\bar{\bm q}$ untrainable and reducing local attention to mean pooling. We benchmark all settings on our 8B model using RULER at a 16K context length.

Table~\ref{tab:local-attn-sigma-full} shows that removing local attention degrades performance. A small $\sigma$ makes the prior overly strong and leads to suboptimal results, while $\sigma$ becomes relatively stable once sufficiently large. We therefore set $\sigma = 10^8$, which weakens the priors while preserving the differentiability.

\section{Full Algorithms}
\label{sec:supp-algorithms}

Algorithms~\ref{alg:dashattention_stage0}--\ref{alg:dashattention_stage2} give the overall view for the three stages of \methodname. We write $\bm{X}_i^{\mathcal{G}_r, j}$ for the slab of a tensor $\bm{X}$ obtained by fixing token $i$, restricting the head axis to $\mathcal{G}_r$, and taking the $j$-th block along the chunk axis; e.g.\ $\bm{Z}_i^{\mathcal{G}_r, j} \in \mathbb{R}^{g_q \times 32}$ is the Stage~1 SRAM tile. The chunk-axis block width is fixed at $32$ because the mask is bit-packed into $\mathrm{int}32$ words.

\begin{algorithm}[h]
\caption{Stage 0: Local Chunk Summarization}
\label{alg:dashattention_stage0}
\begin{algorithmic}[1]
\REQUIRE Per-head summary queries $\bar{\bm{q}}\in\mathbb{R}^{T_c\times h_{kv}\times d_h}$ (RoPE pre-applied), keys $\bm{K}\in\mathbb{R}^{n\times h_{kv}\times d_h}$, chunk size $B$, sequence length $n$.
\ENSURE Chunk summaries $\bar{\bm{K}}\in\mathbb{R}^{T_c\times h_{kv}\times d_h}$.
\STATE Divide $\bm{K}$ along the sequence axis into $T_c=\lfloor n/B\rfloor$ blocks $\bm{K}_1,\dots,\bm{K}_{T_c}$, each of size $B\times h_{kv}\times d_h$.
\FOR{$c=1,\dots,T_c$ and $r=1,\dots,h_{kv}$ in parallel}
    \STATE Load $\bar{\bm{q}}_c^{(r)}\in\mathbb{R}^{d_h}$ and $\bm{K}_c^{(r)}\in\mathbb{R}^{B\times d_h}$ from HBM to SRAM.
    \STATE On chip, compute $\bm{s}_c^{(r)}=\bm{K}_c^{(r)}\bar{\bm{q}}_c^{(r)}/\sqrt{d_h}\in\mathbb{R}^{B}$.
    \STATE On chip, compute $\bar{\bm{k}}_c^{(r)}=\mathrm{softmax}(\bm{s}_c^{(r)})^\top\bm{K}_c^{(r)}\in\mathbb{R}^{d_h}$.
    \STATE Write $\bar{\bm{k}}_c^{(r)}$ to HBM.
\ENDFOR
\end{algorithmic}
\end{algorithm}

\begin{algorithm}[t]
\caption{Stage 1: \entmax Block Routing}
\label{alg:dashattention_stage1}
\begin{algorithmic}[1]
\REQUIRE $\bm{Q}\in\mathbb{R}^{n\times h_q\times d_h}$, $\bar{\bm{K}}\in\mathbb{R}^{T_c\times h_{kv}\times d_h}$; $\alpha>1$, scaling $\gamma$, prior strength $\sigma$.
\ENSURE Bit-packed mask $\bm{\mathcal{M}}\in\{0,1\}^{n\times h_{kv}\times T_c}$ (stored as $\mathrm{int}32$ words, one bit per chunk), routing bias $\bm{d}\in\mathbb{R}^{n\times h_{kv}\times T_c}$.
\STATE Materialize $\bar{\bm{Z}}\in\mathbb{R}^{n\times h_q\times T_c}$ in HBM with $\bar z_{i,c}^{(h)}=\gamma(\alpha-1)\,\langle \bm{q}_i^{(h)},\bar{\bm{k}}_c^{(r(h))}\rangle/\sqrt{d_h}$.
\STATE Divide $\bar{\bm{Z}}$ along the chunk axis into $T_{tc}=\lceil T_c/32\rceil$ blocks of width $32$ (one $\mathrm{int}32$ mask word each), so that $\bar{\bm{Z}}_i^{\mathcal{G}_r,j}\in\mathbb{R}^{g_q\times 32}$ denotes the $j$-th chunk-block at row $i$, restricted to query-head group $\mathcal{G}_r$. Likewise define block slices $\bm{\mathcal{M}}_i^{(r),j}\in\{0,1\}^{32}$ and $\bm{d}_i^{(r),j}\in\mathbb{R}^{32}$.
\FOR{$1\leq i\leq n$ and $1\leq r\leq h_{kv}$ in parallel}
    \STATE Let $J_i=\lceil c_i/32\rceil$ (causal blocks only).
    \STATE For each $h\in\mathcal{G}_r$, solve $f(\tau)=\sum_{c<c_i}\max(0,\bar z_{i,c}^{(h)}-\tau)^{\nicefrac{1}{\alpha-1}}-1=0$ for $\tau_i^{(h)}$ via AdaSplash-2's algorithm~\citep{goncalves2025adasplash}.
    \STATE Initialize $S_i^{(r)}\gets 0,\ N_i^{(r)}\gets 0$ on SRAM.
    \FOR{$j=1$ to $J_i$}
        \STATE Load $\bar{\bm{Z}}_i^{\mathcal{G}_r,j}\in\mathbb{R}^{g_q\times 32}$ from HBM to SRAM.
        \STATE On chip, compute $\bm{w}_i^{(r),j}=\tfrac{1}{g_q}\sum_{h\in\mathcal{G}_r}\max(0,\bar{\bm{z}}_i^{(h),j}-\tau_i^{(h)})^{\nicefrac{1}{\alpha-1}}\in\mathbb{R}^{32}$.
        \STATE On chip, $\bm{\mathcal{M}}_i^{(r),j}\gets\mathds{1}[\bm{w}_i^{(r),j}>0]$.
        \STATE $S_i^{(r)}\mathrel{+}=\sum_{c:{w}_{i,c}^{(r),j}>0}\log{w}_{i,c}^{(r),j}$,\quad $N_i^{(r)}\mathrel{+}=\|\bm{\mathcal{M}}_i^{(r),j}\|_1$.
        \STATE Bit-pack $\bm{\mathcal{M}}_i^{(r),j}$ into one $\mathrm{int}32$ word and write to HBM.
    \ENDFOR
    \STATE $\mu_i^{(r)}\gets S_i^{(r)}/N_i^{(r)}$.
    \FOR{$j=1$ to $J_i$}
        \STATE Reload $\bar{\bm{Z}}_i^{\mathcal{G}_r,j}$ and recompute $\bm{w}_i^{(r),j}$ on chip.
        \STATE On chip, $\bm{d}_i^{(r),j}\gets(\log{\bm{w}}_i^{(r),j}-\mu_i^{(r)})/\sigma$ on active entries, else $0$.
        \STATE Write $\bm{d}_i^{(r),j}$ to HBM.
    \ENDFOR
    \STATE Set $\mathcal{M}_{i,c_i}^{(r)}\gets 1,\ d_{i,c_i}^{(r)}\gets 0$ (diagonal chunk).
\ENDFOR
\end{algorithmic}
\end{algorithm}

\begin{algorithm}[t]
\caption{Stage 2: Prior-Induced Sparse Softmax Attention}
\label{alg:dashattention_stage2}
\begin{algorithmic}[1]
\REQUIRE $\bm{Q}\in\mathbb{R}^{n\times h_q\times d_h}$, $\bm{K},\bm{V}\in\mathbb{R}^{n\times h_{kv}\times d_h}$, per-token mask $\bm{\mathcal{M}}\in\{0,1\}^{n\times h_{kv}\times T_c}$ and routing bias $\bm{d}\in\mathbb{R}^{n\times h_{kv}\times T_c}$ from Stage~1, chunk size $B$, GQA group size $g_q$.
\ENSURE $\bm{O}\in\mathbb{R}^{n\times h_q\times d_h}$.
\STATE Divide $\bm{K}$ and $\bm{V}$ along the sequence axis into $T_c=\lceil n/B\rceil$ blocks $\bm{K}_1,\dots,\bm{K}_{T_c}$, $\bm{V}_1,\dots,\bm{V}_{T_c}$, each of size $B\times h_{kv}\times d_h$.
\FOR{$i=1,\dots,n$ and $r=1,\dots,h_{kv}$ in parallel}
    \STATE Load $\bm{Q}_i^{\mathcal{G}_r}\in\mathbb{R}^{g_q\times d_h}$ from HBM to SRAM.
    \STATE Initialize $\bm{O}_i^{\mathcal{G}_r}\gets\bm{0}_{g_q\times d_h}$, $\bm{\ell}_i\gets\bm{0}_{g_q}$, $\bm{m}_i\gets-\infty\,\bm{1}_{g_q}$ on SRAM.
    \FOR{$j$ such that $\mathcal{M}_{i,j}^{(r)}=1$}
        \STATE Load $\bm{K}_j^{(r)},\bm{V}_j^{(r)}\in\mathbb{R}^{B\times d_h}$ and $d_{i,j}^{(r)}\in\mathbb{R}$ from HBM to SRAM.
        \STATE On chip, compute $\bm{S}_i^{\mathcal{G}_r,j}=\bm{Q}_i^{\mathcal{G}_r}\bm{K}_j^{(r)\top}/\sqrt{d_h}+d_{i,j}^{(r)}\,\bm{1}_{g_q\times B}\in\mathbb{R}^{g_q\times B}$.
        \STATE If $j=c_i$: apply causal mask to $\bm{S}_i^{\mathcal{G}_r,j}$.
        \STATE $\bm{m}_i'\gets\max(\bm{m}_i,\mathrm{rowmax}(\bm{S}_i^{\mathcal{G}_r,j}))$,\ \ $\bm{P}_i^{\mathcal{G}_r,j}\gets\exp(\bm{S}_i^{\mathcal{G}_r,j}-\bm{m}_i'\bm{1}^\top)$.
        \STATE $\bm{\ell}_i\gets e^{\bm{m}_i-\bm{m}_i'}\odot\bm{\ell}_i+\mathrm{rowsum}(\bm{P}_i^{\mathcal{G}_r,j})$.
        \STATE $\bm{O}_i^{\mathcal{G}_r}\gets\mathrm{diag}(e^{\bm{m}_i-\bm{m}_i'})\,\bm{O}_i^{\mathcal{G}_r}+\bm{P}_i^{\mathcal{G}_r,j}\bm{V}_j^{(r)}$.
        \STATE $\bm{m}_i\gets\bm{m}_i'$.
    \ENDFOR
    \STATE $\bm{O}_i^{\mathcal{G}_r}\gets\mathrm{diag}(\bm{\ell}_i)^{-1}\bm{O}_i^{\mathcal{G}_r}$; write $\bm{O}_i^{\mathcal{G}_r}$ to HBM.
\ENDFOR
\end{algorithmic}
\end{algorithm}

\end{document}